\documentclass{article}
\usepackage{arxiv}
\renewcommand{\today}{September 30, 2026}

\usepackage[utf8]{inputenc} %
\usepackage[T1]{fontenc}    %
\usepackage{url}            %
\usepackage{booktabs}       %
\usepackage{amsfonts}       %
\usepackage{amsmath}        %
\usepackage{mathtools}      %
\usepackage{mathdots}       %
\usepackage{nicefrac}       %
\usepackage{graphicx}       %
\usepackage{float}          %
\usepackage{bm}             %
\usepackage{lineno}         %
\usepackage{todonotes}      %
\usepackage{pgfplotstable}  %
\usepackage{pgfplots}
\pgfplotsset{compat=1.18}
\usepgfplotslibrary{groupplots}
\usepackage[numbers]{natbib}
\usepackage{tcolorbox}      %
\tcbuselibrary{breakable}

\usepackage{amsthm}

\theoremstyle{definition}
\newtheorem{definition}{Definition}[section]

\theoremstyle{plain}
\newtheorem{theorem}{Theorem}[section]

\theoremstyle{plain}

\theoremstyle{remark}
\newtheorem*{remark}{Remark}

\usepackage[ruled,vlined]{algorithm2e}
\DontPrintSemicolon

\usepackage{hyperref}       %
\hypersetup{ colorlinks=false, pdfborder={0 0 0} }
\usepackage{cleveref}       %
\crefname{algocf}{algorithm}{algorithms}
\Crefname{algocf}{Algorithm}{Algorithms}

\AddToHook{cmd/appendix/before}{\crefalias{section}{appendix}}
\AddToHook{cmd/appendix/before}{\crefalias{subsection}{appendix}}

\crefname{subsection}{section}{sections}
\Crefname{subsection}{Section}{Sections}

\crefname{appendix}{Appendix}{Appendices}
\Crefname{appendix}{Appendix}{Appendices}

\graphicspath{ {./figures/} }

\DeclareMathOperator{\tr}{tr}

\DeclareMathOperator{\diag}{diag}
\DeclareMathOperator{\blkdiag}{blkdiag}

\DeclareMathOperator{\cov}{cov}
\DeclareMathOperator{\chol}{chol}

\newcommand{\vh}{\mathbf{h}}    %
\newcommand{\vo}{\mathbf{y}}    %
\newcommand{\Mo}{\mathbf{C}}    %
\newcommand{\Mf}{\mathbf{D}}    %

\newcommand{\naturals}{\mathbb{N}}

\newcommand{\reals}{\mathbb{R}}
\newcommand{\integers}{\mathbb{Z}}

\newcommand{\dr}{\ensuremath{\mathrm{d}}}

\newcommand{\eu}{\ensuremath{\mathrm{e}}}

\newcommand{\du}{\ensuremath{\mathrm{d}}}

\newcommand{\cD}{\mathcal{D}}
\newcommand{\cI}{\mathcal{I}}

\newcommand{\cN}{\mathcal{N}}
\newcommand{\cO}{\mathcal{O}}
\newcommand{\cS}{\bm{\mathcal{S}}}

\newcommand{\boD}{\overline{\bD}}

\newcommand{\boT}{\overline{\bT}}
\newcommand{\boU}{\overline{\bU}}
\newcommand{\boV}{\overline{\bV}}

\newcommand{\btD}{\tilde{\bD}}
\newcommand{\btF}{\tilde{\bF}}
\newcommand{\btP}{\tilde{\bP}}

\newcommand{\btV}{\tilde{\bV}}

\newcommand{\bbD}{\breve{\bD}}

\newcommand{\bbP}{\breve{\bP}}

\newcommand{\bbR}{\breve{\bR}}

\newcommand{\bbT}{\breve{\bT}}
\newcommand{\bbU}{\breve{\bU}}
\newcommand{\bbV}{\breve{\bV}}
\newcommand{\bbQ}{\breve{\bQ}}

\newcommand{\bA}{\mathbf{A}}
\newcommand{\bB}{\mathbf{B}}
\newcommand{\bC}{\mathbf{C}}
\newcommand{\bD}{\mathbf{D}}
\newcommand{\bM}{\mathbf{M}}
\newcommand{\bN}{\mathbf{N}}
\newcommand{\bQ}{\mathbf{Q}}
\newcommand{\bK}{\mathbf{K}}
\newcommand{\bG}{\mathbf{G}}

\newcommand{\bP}{\mathbf{P}}

\newcommand{\bT}{\mathbf{T}}
\newcommand{\bV}{\mathbf{V}}
\newcommand{\bL}{\mathbf{L}}
\newcommand{\bE}{\mathbf{E}}
\newcommand{\bX}{\mathbf{X}}

\newcommand{\bI}{\mathbf{I}}
\newcommand{\bW}{\mathbf{W}}
\newcommand{\bR}{\mathbf{R}}

\newcommand{\bU}{\mathbf{U}}
\newcommand{\bZ}{\mathbf{Z}}
\newcommand{\bH}{\mathbf{H}}
\newcommand{\bJ}{\mathbf{J}}
\newcommand{\bF}{\mathbf{F}}

\newcommand{\bLambda}{\bm{\Lambda}}
\newcommand{\bSigma}{\bm{\Sigma}}
\newcommand{\by}{\mathbf{y}}

\newcommand{\bx}{\mathbf{x}}

\newcommand{\bu}{\mathbf{u}}

\newcommand{\bb}{\mathbf{b}}
\newcommand{\bc}{\mathbf{c}}

\newcommand{\be}{\mathbf{e}}
\newcommand{\bff}{\mathbf{f}}

\newcommand{\bh}{\mathbf{h}}

\newcommand{\bz}{\mathbf{z}}

\newcommand{\bt}{\mathbf{t}}
\newcommand{\bv}{\mathbf{v}}
\newcommand{\bw}{\mathbf{w}}

\newcommand{\bzero}{\mathbf{0}}

\newcommand{\btheta}{\bm{\theta}}
\newcommand{\bTheta}{\bm{\Theta}}

\newcommand{\bUps}{\mathbf{\Upsilon}}

\newcommand{\nhyper}{n_{\mathrm{h}}}
\newcommand{\ntest}{n_{\mathrm{t}}}

\DeclareMathOperator*{\argmax}{argmax}

\title{On the Sequentially Semiseparable Structure of SDE-Induced Kernels and Efficient Algorithms for Kernel-Based Estimation Problems}
\author{
    Mikkel Paltorp \\
    Department of Applied Mathematics and Computer Science\\
    Technical University of Denmark\\
    \texttt{mpasc@dtu.dk}
    \And
    Luigi Caglio \\
    Department of Civil and Mechanical Engineering\\
    Technical University of Denmark\\
    \texttt{lucag@dtu.dk}
    \And
    Tianshi Chen\\
    School of Data Science and Shenzhen Research Institute of Big Data\\
    The Chinese University of Hong Kong, Shenzhen, 58182, China\\
    \texttt{tschen@cuhk.edu.cn}
    \And
    Evangelos Katsanos\\
    Department of Civil and Mechanical Engineering\\
    Technical University of Denmark\\
    \texttt{vakat@dtu.dk}
    \And
    Martin S. Andersen \\
    Department of Applied Mathematics and Computer Science\\
    Technical University of Denmark\\
    \texttt{mskan@dtu.dk}
}

\newcommand{\journallayout}[2]{#2}

\newcommand{\journalmultline}[2]{%
  \journallayout{%
    \begin{multline*}
      #1\\
      #2
    \end{multline*}%
  }{%
    \begin{align*}
      #1#2
    \end{align*}%
  }%
}

\newcommand{\layoutidentity}[2]{%
  \journallayout{\underbrace{#2}_{#1}}{#1 = #2}%
}

\newcommand{\journalcompactmath}[1]{%
  \journallayout{%
    \vcenter{\hbox{\footnotesize\setlength{\arraycolsep}{3pt}%
      $\displaystyle #1$}}%
  }{#1}%
}

\begin{document}
\maketitle

\begin{abstract}
	We study kernel functions induced by a linear time-invariant (LTI) stochastic differential equation (SDE), an approach that embeds prior knowledge about the function to be estimated directly into the state-space description of the SDE. This kernel-design perspective is relevant across the many settings where kernels play a central role, including kernel methods in statistics, Gaussian process (GP) regression in machine learning, and kernel-based regularization methods in system identification. The resulting class of SDE-induced kernels includes several kernels of practical interest, such as Matérn kernels with half-integer smoothness, spline kernels, and related kernels used in system identification. In contrast to existing linear-time GP regression methods for these kernels, which are typically based on Kalman filtering and Rauch--Tung--Striebel smoothing on the latent state, we work directly with the SSS representation of the observed covariance matrix. This yields linear-complexity recursive algorithms for the matrix operations required in GP regression, including factorization, linear solves, log determinants, posterior covariance, log-marginal-likelihood gradients, and posterior sampling; the same generator-level recursions extend, with only minor modifications, to the analogous computations in kernel-based regularization methods for system identification. Compared with previous uses of SSS representations in GP regression, we derive efficient recursive algorithms for analytic marginal-likelihood gradients in three application-relevant cases, and we develop new recursions for posterior computations and Matheron-type sampling. Numerical experiments verify the accuracy of the proposed recursions and illustrate posterior inference and sampling for representative SDE-induced kernels.
\end{abstract}


\section{Introduction}\label{sec:introduction}

Kernels are the central object in a broad range of estimation problems, including kernel methods in statistics \cite{Wah:90,KiW:70}, kernel methods in machine learning \cite{HSS:08}, most notably Gaussian process (GP) regression \cite{rasmussen2006a}, and kernel-based regularization methods for system identification \cite{pillonetto2014a,pillonetto2022a,chen2018a}. Across these fields, kernel design is a crucial step: the kernel function should be parameterized such that it incorporates suitable prior knowledge about the function being estimated while, ideally, inducing a kernel matrix with favorable mathematical structure that enables efficient computation. Kernels can be designed in many ways; in this paper, we consider kernels that are induced by a linear time-invariant (LTI) stochastic differential equation (SDE). This construction achieves both design goals at once: prior knowledge about the function to be estimated is embedded in the state-space matrices of the SDE, and the resulting kernel matrix is guaranteed to admit a sequentially semiseparable (SSS) representation for ordered one-dimensional inputs. This SSS structure, in turn, enables linear-time algorithms for the matrix computations that underlie kernel-based estimation.

GP regression with one-dimensional inputs is a setting in which the connection between SDEs and Markovian state-space realizations has been studied extensively. Classical approaches exploit this connection by converting the regression problem into Kalman filtering and Rauch--Tung--Striebel (RTS) smoothing on the latent state \cite{rauch1965a,hartikainen2010a,sarkka:sde,Rubenzahl2026}, giving linear complexity for important kernels such as Matérn kernels with half-integer smoothness and spline-related kernels. This perspective, however, treats the SDE model chiefly as a computational device for GP inference. The complementary perspective adopted in this paper is that the SDE is itself a kernel design tool: it specifies how prior structural knowledge about the process, such as smoothness or decay rate, is encoded in the state-space representation of the SDE, and it simultaneously fixes the rank structure of the resulting kernel matrix. Kernel design from prior structural or physical knowledge is a long-standing theme in system identification \cite{chen2018a,pillonetto2014a,pillonetto2022a,chen2021a}, but, to our knowledge, LTI SDE models have not previously been exploited as an explicit route to structured kernel matrices outside of the state-space GP regression literature.

We show that the covariance matrix of the observed process itself admits an SSS representation, and that this remains true even when the corresponding precision matrix is not block tridiagonal. The SSS representation therefore provides a direct structured-matrix route to kernel-based estimation. In particular, we show that the standard matrix computations required for GP regression can all be expressed as linear-time recursions on SSS generators.

We start by formally defining the class of kernels that we consider in this paper.
\begin{definition}[LTI-SDE-induced kernel]\label{def:lti-sde-kernel}
A function $k\colon \reals \times \reals \to \reals^{d\times d}$ is an \emph{LTI-SDE-induced kernel} if it is the matrix-valued covariance function of a stochastic process $\by(t)$ admitting a state-space representation of the form \cite{sarkka:sde,hartikainen2010a}
\begin{subequations}\label{eq:sde}
	\begin{align}
	\label{eq:sde:state}
		\frac{\dr \bx(t)}{\dr t} & =\bF\bx(t)+\bG\bw(t), \\
		\by(t)                   & =\bC(t)\bx(t),
	\end{align}
\end{subequations}
where $\bF\in \reals^{p\times p}$, $\bG\in \reals^{p\times r}$, and $\Mo(t)\in \reals^{d\times p}$. Here, $\bw(t)\in\reals^{r}$ denotes zero-mean Gaussian white noise, i.e., the formal derivative of a Brownian motion with diffusion matrix $\bQ_{\bw} \in \reals^{r\times r}$, so that $\mathbb{E}[\bw(t)\bw(t')^\top] = \bQ_{\bw}\,\delta(t-t')$.
\end{definition}
\begin{remark}
The state equation in \eqref{eq:sde:state} is understood as shorthand for the corresponding Itô SDE
	\begin{align*}
	\du \bx(t) & = \bF\bx(t)\du t + \bG \bQ_{\bw}^{1/2} \du \bW(t),
\end{align*}
where $\bW(t)$ is a standard $r$-dimensional Brownian motion. The process $\by(t)$ is then a zero-mean Gaussian process with kernel function $k$.
\end{remark}
The model \eqref{eq:sde} will often depend on a vector of $\nhyper $ hyperparameters, which will be denoted by $\btheta = (\theta_1,\ldots,\theta_{\nhyper})$, in which case we write
\[ k(t,t';\btheta) = \cov(\by(t;\btheta),\by(t';\btheta)), \]
where $\cov(\bu,\bv) = \mathbb{E}[\bu\bv^\top] -\mathbb{E}[\bu]\mathbb{E}[\bv^\top]$ denotes the cross-covariance of vector random variables $\bu$ and $\bv$. To simplify notation, we will frequently omit $\btheta$. Throughout, the term \emph{state-space representation} refers to the latent state-space model \eqref{eq:sde} that defines the process $\by(t)$ and induces the kernel $k$.

Given a nondecreasing sequence of inputs $(t_1,\ldots,t_n)$, we define the covariance matrix $\bK \in \reals^{nd \times nd}$ with $d\times d$ blocks $\bK_{ij} = k(t_i,t_j)$. We will show that $\bK$ admits an SSS representation of the form
\[
	\bK = \bU(\bI-\bN^p\bT)^{-1}\bN^p\bV^\top + \bV(\bN^p)^\top(\bI-\bN^p\bT)^{-\top}\bU^\top + \bD,
\]
where $\bU,\bV,\bT$, and $\bD$ are block-diagonal matrices and $\bN$ is the lower shift matrix such that $\bN^p$ is a Toeplitz matrix with ones on the $p$th subdiagonal and zeros elsewhere. The matrix $\bK$ can be expanded as a block matrix with blocks $\bK_{ij}$ given by
\begin{align}\label{eq:sss:symmetric}
	\bK_{ij} = \begin{cases}
		           \bU_i\bT_{i-1:j+1}\bV_j^\top,        & i>j, \\
		           \bD_i,                               & i=j, \\
		           \bV_i(\bT_{j-1:i+1})^\top\bU_j^\top, & i<j,
	           \end{cases}
\end{align}
where $\bU_i \in \reals^{d \times p},\bV_i\in \reals^{d \times p},\bT_i\in\reals^{p \times p}$, and $\bD_i\in \reals^{d \times d}$ denote the blocks of $\bU,\bV,\bT$, and $\bD$, respectively, and $\bT_{i:j}$ denotes the descending-index chain product of square matrices $\bT_i,\ldots,\bT_j$, defined as
\begin{equation}\label{eq:chain:product}
	\bT_{i:j} =
	\begin{cases}
		\bT_{i}\bT_{i-1}\cdots\bT_j, & i > j, \\
		\bT_i,                       & i = j, \\
		\bI,                         & i < j.
	\end{cases}
\end{equation}
\Cref{sec:sde-to-sss} uses only this symmetric covariance form to connect \eqref{eq:sde} and \eqref{eq:sss:symmetric}. We note that in the special case where $\bT_1,\ldots,\bT_n$ are all identity matrices, the SSS representation \eqref{eq:sss:symmetric} reduces to a generator representation of a semiseparable matrix \cite{andersen2020a,golub13}. A more general nonsymmetric SSS representation, which is useful for certain matrix-vector products and trace computations in GP regression, will be introduced in \Cref{sec:background}.

This work is related to several existing approaches for scalable temporal GP regression and kernel-based estimation. A state-space formulation of a GP provides a natural way to perform inference by exploiting the finite-dimensional state-space representation of the underlying stochastic process via Kalman filtering and RTS smoothing. Rank-structured covariance methods instead work directly with generator-representable semiseparable kernel matrices, as in celerite \cite{celerite} and related algorithms \cite{andersen2020a,shen2025}. Sparse-precision approaches use the Markov structure of the latent state process, leading to banded or block-tridiagonal precision matrices \cite{meurant1992,grigorievskiy2017,chen2018a}. A subtle but important distinction is that the Markov property applies to the latent state $\bx(t)$, not necessarily to the observed process $\by(t)=\bC(t)\bx(t)$. After marginalizing or projecting the state through $\bC(t)$, the covariance matrix of $\by(t_1),\ldots,\by(t_n)$ remains structured, but its inverse need not be block tridiagonal; see \Cref{app:block:tridiagonal}. Hence sparse-precision methods based on the latent state do not directly apply in this setting. A parallel line of work in system identification designs kernel matrices directly and separately exploits their semiseparable structure for regularized estimation \cite{chen2018a,chen2021a,andersen2020a}, without deriving the kernel from an explicit SDE. The present paper connects these perspectives: SDE-based kernel design supplies the state-space parameterization $(\bF,\bG,\bC)$ used throughout state-space GP regression, control, and system identification, while the SSS representation of the resulting kernel matrix in \eqref{eq:sss:symmetric} supplies the structure that both GP regression and kernel-based regularization methods can exploit for efficient computation.

The proposed approach is not intended to replace Kalman filtering and RTS smoothing as a state-space inference method. Rather, it provides an equivalent matrix formulation in which the standard kernel-based estimation primitives are expressed directly as structured matrix operations. This perspective is particularly convenient for deriving marginal-likelihood gradients, selected posterior variances, and Matheron-type posterior sampling rules.

\paragraph{Contributions} The main contributions of this paper are as follows:

\begin{itemize}
	\item \textbf{SDE-based kernel design and SSS representations of output covariance matrices:} We show how covariance matrices generated by observed outputs of continuous-time LTI SDEs admit block-SSS representations for ordered one-dimensional inputs, covering both stationary and nonstationary kernels. This connects the SDE-based kernel design perspective used in state-space GP regression with the semiseparable structure exploited by kernel-based regularization methods for system identification.
	\item \textbf{Efficient recursive algorithms:} We derive linear-time SSS recursions for the GP matrix primitives required for training and inference: factorization, triangular solves, log determinants, selected inverse diagonals, and structured matrix-vector products. These algorithms generalize existing algorithms for generator-representable semiseparable matrices to the more general SSS setting and circumvent known numerical issues associated with the generator representation of semiseparable matrices \cite{celerite,andersen2020a}.
	\item \textbf{Novel trace computation:} We derive a trace recursion for $\tr((\bK + \bSigma)^{-1}\bA)$ when $\bK$ and $\bA$ are given in SSS form and $\bSigma$ is block diagonal and symmetric positive definite, enabling analytic marginal-likelihood gradients without forming dense matrices or intermediate SSS products.
	\item \textbf{Scalable inference:} We show how the same representation supports posterior mean and variance computations for interleaved training and test points, as well as posterior sampling through Matheron's update rule.
	\item \textbf{Examples:} We demonstrate the resulting framework on LTI-SDE-induced kernels, including
	      Matérn, spline, celerite, DC, TC, and stable spline kernels.
\end{itemize}

\paragraph{Outline} \Cref{sec:background} outlines notation and reviews the necessary background on SSS matrices. \Cref{sec:sde-to-sss} shows how SDE-induced kernels give rise to SSS covariance matrices and
illustrates the construction for stationary, nonstationary, and transformed
kernels. \Cref{sec:algorithms} derives the linear-time SSS algorithms, which are used in \Cref{sec:gp:regression} as building blocks for GP training, inference, and sampling routines. \Cref{sec:numerical_experiments} presents numerical experiments, and \Cref{sec:conclusion} concludes.

\section{Background}\label{sec:background}
We now introduce the notation and summarize the key properties of SSS matrices that will be used in the subsequent sections. We also provide a brief review of a general SSS matrix representation that we will need for some of the GP primitives.

\subsection{Notation}\label{sub:sec:notation}
The set $\naturals$ denotes the natural numbers $\{1,2,\ldots\}$, $\integers$ denotes the set of integers, and $\integers_+$ is the set of nonnegative integers. Similarly, $\reals^n$ is the standard Euclidean space of dimension $n$, and $\reals^{m\times n}$ is the set of $m\times n$ real matrices. The vector $\be_k$ denotes the $k$th canonical basis vector, with dimension inferred from context, and $\bI_n$ denotes the identity matrix of order $n$, or simply $\bI$ when the order can be inferred from context.

Given a sequence $(t_1, \ldots, t_n)$ of real numbers, we define $\Delta t_i = t_{i+1} - t_i$ for $i \in \{1,\ldots,n-1\}$ and, for ease of notation, we define $\Delta t_n = 0$. Given $n$ vectors $\bx_1,\dots, \bx_n \in \reals^d$, we write their vertical concatenation as $\bx = (\bx_1, \dots, \bx_n) \in \reals^{nd}$. Similarly, for matrices $\bX_1,\dots, \bX_n \in \reals^{r\times s}$, we denote by $\bX = \blkdiag(\bX_1, \dots, \bX_n) \in \reals^{nr\times ns}$ the block-diagonal matrix with diagonal blocks $\bX_1,\dots,\bX_n$.

The square matrix $\bN_n$ denotes the down-shift matrix of order $n$, which has ones on the first subdiagonal and zeros elsewhere. The subscript will be dropped when the order can be inferred from context. We will frequently use the block down-shift identity $\bN_n \otimes \bI_s = \bN_{ns}^s$, where $\otimes$ denotes the Kronecker product. Thus, if $\bX = \blkdiag(\bX_1,\ldots,\bX_n)$ for $\bX_1,\ldots,\bX_n \in \reals^{s \times s}$, then the block-matrices $\bI - \bN_{ns}^s\bX$ and $(\bI - \bN_{ns}^s\bX)^{-1}\bN_{ns}^s$ are given by
\begin{align}\label{eq:blk:bidiag}
	\journalcompactmath{
	\layoutidentity
	{\bI - \bN^s\bX}
	{\begin{bmatrix}
		\bI                                         \\
		-\bX_1 & \bI                                \\
		       & -\bX_2 & \bI                       \\
		       &        & \ddots & \ddots           \\
		       &        &        & -\bX_{n-1} & \bI
	\end{bmatrix}},\quad
	\layoutidentity
	{(\bI - \bN^s\bX)^{-1}\bN^s}
	{\begin{bmatrix}
		\bzero                                               \\
		\bI         & \bzero                                 \\
		\bX_{2:2}   & \bI         & \ddots                   \\
		\vdots      & \vdots      & \ddots & \ddots          \\
		\bX_{n-1:2} & \bX_{n-1:3} & \cdots & \bI    & \bzero
	\end{bmatrix}}.
	}
\end{align}

\subsection{Sequentially semiseparable matrices}

Semiseparable and sequentially semiseparable matrices are classes of rank-structured matrices. Our main focus will be the class of SSS matrices. This class and the associated numerical linear algebra routines have been studied extensively in the numerical linear algebra community and have been applied to a wide range of problems, including integral equations and time-variant systems \cite{dewilde2024a,chandrasekaran2003b,eidelman2014a,golub13,chandrasekaran2003a,chandrasekaran2005a}.
SSS matrices also arise in application areas such as automatic control, machine learning, and scientific computing, for example, in Gaussian process regression for system identification \cite{andersen2020a,celerite,celerite2,shen2025} and modern deep learning sequence models \cite{mamba1,mamba2,mamba3}.

A square matrix $\bA \in\reals^{nd\times nd}$ is said to be SSS if there exist integers $p,q\geq 1$ and block-diagonal matrices $\bT \in \reals^{np \times np}$, $\bU,\bV\in \reals^{nd \times np}$, $\bR \in \reals^{nq\times nq}$, $\bP,\bQ \in \reals^{nd\times nq}$, and $\bD \in \reals^{nd\times nd}$, with blocks
\begin{align*}
	\bT_i       & \in \reals^{p\times p}, &
	\bU_i,\bV_i & \in \reals^{d\times p}, &
	\bR_i       & \in \reals^{q\times q}, &
	\bP_i,\bQ_i & \in \reals^{d\times q}, &
	\bD_i       & \in \reals^{d\times d},
\end{align*}
such that \cite{dewilde2024a,chandrasekaran2003b}
\begin{equation}\label{eq:def:nonsym:A}
	\bA = \bU (\bI - \bN^p\bT)^{-1}\bN^p\bV^\top + \bP(\bN^q)^\top (\bI - \bN^q\bR)^{-\top}\bQ^\top + \bD.
\end{equation}
The integers $p$ and $q$ are called the lower and upper order of the representation. We note that $\bU_1$, $\bT_1$, $\bT_n$, $\bV_n$, $\bQ_1$, $\bR_1$, $\bR_n$, and $\bP_n$ are included for notational convenience and consistency with state-space models; these matrices do not have any direct influence on $\bA$ in \eqref{eq:def:nonsym:A}.
Indeed, the full structure of the SSS representation expands to
\begin{equation}\label{eq:full:nonsym:A}
	\journalcompactmath{
	\bA =
	\begin{bmatrix}
		\bD_1                         & \bP_1\bQ_2^\top             & \bP_1\bR_2^\top\bQ_3^\top & \cdots & \bP_1\bR_{n-1:2}^\top\bQ_n^\top \\
		\bU_2\bV_1^\top               & \bD_2                       & \bP_2\bQ_3^\top           & \cdots & \bP_2\bR_{n-1:3}^\top\bQ_n^\top \\
		\bU_3\bT_2\bV_1^\top          & \bU_3\bV_2^\top             & \bD_3                     & \cdots & \bP_3\bR_{n-1:4}^\top\bQ_n^\top \\
		\vdots                        & \vdots                      & \vdots                    & \ddots & \vdots \\
		\bU_n\bT_{n-1:2}\bV_1^\top    & \bU_n\bT_{n-1:3}\bV_2^\top & \bU_n\bT_{n-1:4}\bV_3^\top & \cdots & \bD_n
	\end{bmatrix}.
	}
\end{equation}

The SSS representation can be stored efficiently by storing its generators $\bT_i$, $\bU_i$, $\bV_i$, $\bR_i$, $\bP_i$, $\bQ_i$, and $\bD_i$ for $i\in \{1,\ldots,n\}$, which requires $\cO(n(d(p+q) + d^2 + p^2 + q^2))$ storage. Moreover, key linear algebra operations such as matrix-vector multiplication, triangular solves, and factorizations can be performed efficiently with linear complexity with respect to $n$ \cite{dewilde2024a,vandebril2007a,shen2025,eidelman1999a,eidelman2014a}. The structure of the SSS representation may be linked to forward and backward recursions \cite{dewilde2024a}. This is readily seen from the multiplication with the strictly lower triangular part, i.e., $\by^L=\bU (\bI - \bN^p\bT)^{-1}\bN^p\bV^\top\bx$, which may be viewed as a two-step process where $\bh = (\bI - \bN^p\bT)^{-1}\bN^p\bV^\top\bx$ is first computed via forward substitution and then scaled by $\bU$, i.e.,
\begin{align*}
\journalcompactmath{
	\underbrace{\begin{bmatrix}
			\bI                                \\
			-\bT_1 & \bI                       \\
			       & \ddots & \ddots           \\
			       &        & -\bT_{n-1} & \bI
		\end{bmatrix}}_{\bI - \bN^p\bT}
	\underbrace{\begin{bmatrix}\vh_1 \\  \vh_2 \\ \vdots \\ \vh_n \end{bmatrix}}_{\vh}
	\journallayout{}{}=
	\underbrace{\begin{bmatrix}
			\bzero                                        \\
			\bV_1^\top & \bzero                           \\
			           & \ddots & \ddots                  \\
			           &        & \bV_{n-1}^\top & \bzero
		\end{bmatrix}}_{\bN^p\bV^{\top}}
	\underbrace{\begin{bmatrix}  \bx_1 \\ \bx_2 \\ \vdots \\ \bx_n \end{bmatrix}}_{\bx}, \journallayout{\ \ }{\quad}
	\underbrace{\begin{bmatrix} \vo_1 \\ \vo_2 \\ \vdots \\ \vo_n \end{bmatrix}}_{\vo}
	\journallayout{}{}=
	\underbrace{
		\begin{bmatrix}
			\bU_1                     \\
			 & \bU_2                  \\
			 &       & \ddots         \\
			 &       &        & \bU_n
		\end{bmatrix}}_{\bU}\vh.
}
\end{align*}
An analogous recursion based on backward substitution yields matrix-vector products with the strictly block-upper-triangular part of $\bA$, namely $\by^U = \bP(\bN^q)^\top(\bI - \bN^q\bR)^{-\top}\bQ^\top\bx$. Thus, the matrix-vector multiplication with $\bA$ can be computed efficiently as $\by = \bA\bx = \by^L + \by^U + \bD\bx$.

A well-known result is that the triangular factors of an LDU factorization of an SSS matrix $\bA$ are also SSS with the same transition and output matrices, i.e., $\bA$ has an LDU factorization of the form \cite{eidelman2014a}
\begin{equation}
	\bA = \left[\bU (\bI - \bN^p\bT)^{-1}\bN^p\btV^\top + \bI\right]\btD\left[\bQ (\bI - \bN^q\bR)^{-1}\bN^q\btP^\top + \bI\right]^\top,
\end{equation}
where $\btV$, $\btD$, and $\btP$ are block-diagonal matrices. Similar results hold for LU, LDL$^\top$, and Cholesky factorizations; we derive the Cholesky factorization in \Cref{sub:sec:cholesky} and the LDL$^\top$ factorization in \Cref{app:sec:ldl}.
Another useful property of SSS matrices is that the class is closed under inversion; that is, the inverse of an SSS matrix is also SSS \cite{eidelman1999a}. This implies that the inverse of a lower triangular SSS matrix remains SSS (\Cref{sub:sec:explicit:inverse}). Specifically, for a nonsingular, lower-triangular SSS matrix $\bL=\bU (\bI - \bN^p\bT)^{-1}\bN^p\bV^\top + \bD$, it holds that
\begin{equation}
	\bL^{-1} = \boU (\bI - \bN^p\boT)^{-1}\bN^p\boV^\top + \bD^{-1},
\end{equation}
where the block-diagonal matrices $\boU, \boT$, and $\boV$ can be computed from $\bU, \bT$, $\bV$, and $\bD$ using \Cref{alg:explicit:inverse}, which is included in \Cref{app:additional:algorithms}.

\section{From SDE-induced kernels to sequentially semiseparable covariance matrices}\label{sec:sde-to-sss}

Recall that the process $\by(t)$ defined by the LTI SDE model in \eqref{eq:sde} is a zero-mean Gaussian process with kernel function
\[k(t,t') = \cov(\by(t),\by(t')) = \bC(t)\cov(\bx(t),\bx(t')) \bC(t')^\top.\]
We will now show that for any nondecreasing sequence $(t_1,\ldots,t_n)$, the symmetric matrix $\bK$ with blocks $\bK_{ij} = k(t_i,t_j)$ admits an SSS representation of the form \eqref{eq:sss:symmetric}.

Without loss of generality, we may assume that the state process $\bx(t)$ is defined for $t\geq t_0$ with initial condition $\bx(t_0)\sim \mathcal{N}(\bzero,\bP_0)$. For $t\geq t'\geq t_0$, the covariance between $\bx(t)$ and $\bx(t')$ then takes the form \cite{sarkka:sde}
\begin{align*}
	\cov(\bx(t),\bx(t')) = \eu^{\bF t}\bP_0\eu^{\bF^\top t'} + \int_{t_0}^{t'} \eu^{\bF (t-\tau)}\bG \bQ_{\bw} \bG^\top\eu^{\bF^\top (t'-\tau)}\, \dr\tau = \eu^{\bF (t-t')} \bP(t'),
\end{align*}
where
\begin{align*}
	\bP(t') = \cov(\bx(t'),\bx(t')) = \eu^{\bF t'}\bP_0\eu^{\bF^\top t'} + \int_{t_0}^{t'} \eu^{\bF (t'-\tau)}\bG \bQ_{\bw} \bG^\top\eu^{\bF^\top (t'-\tau)}\, \dr\tau.
\end{align*}
Thus, the kernel function $k(t,t') = \cov(\by(t),\by(t'))$ can be expressed as
\begin{align}\label{eq:sde:nonstationary:cov}
	k(t,t') = \cov(\bC(t)\bx(t),\bC(t')\bx(t')) = \begin{cases}
		                                              \bC(t)\eu^{\bF (t-t')} \bP(t')\bC(t')^\top,    & t \geq t', \\
		                                              \bC(t)\bP(t)\eu^{\bF^\top (t'-t)}\bC(t')^\top, & t < t'.
	                                              \end{cases}
\end{align}
This leads to the following result, connecting the SDE representation of the process $\by(t)$ to an SSS representation of the covariance matrix $\bK$.
\begin{theorem}[SSS representation of LTI-SDE-induced kernels]\label{thm:sde-to-sss}
    Let $\by(t)$ be a zero-mean Gaussian process defined by the LTI SDE model in \eqref{eq:sde} with kernel function $k(t,t') = \cov(\by(t),\by(t'))$. For any nondecreasing sequence $(t_1,\ldots,t_n)$, the covariance matrix $\bK$ with blocks $\bK_{ij} = k(t_i,t_j)$ admits an SSS representation of the form \eqref{eq:sss:symmetric} with generators
    \begin{align}\label{eq:sss:generators:nonstationary}
\journallayout{\begin{aligned}}{}
        \bU_i & = \bC(t_i), &
        \bT_i & = \eu^{\bF \Delta t_{i}}, \journallayout{\\}{&}
        \bV_i & = \bC(t_i) \bP(t_i) \bT_i^{\top}, &
        \bD_i & = \bC(t_i) \bP(t_i) \bC(t_i)^\top, \journallayout{\qquad}{&} i\in\{1,\ldots,n\}.
    \journallayout{\end{aligned}}{}
\end{align}
\end{theorem}
\begin{proof}
The SSS representation of $\bK$ follows directly from \eqref{eq:sde:nonstationary:cov} and the definition of the chain product in \eqref{eq:chain:product}, which implies that for any $\tau \in [t',t]$, we have $\eu^{\bF(t-t')} = \eu^{\bF (t-\tau)}\eu^{\bF(\tau - t')}$. Using this property, we find that for $i>j$,
\begin{align*}
  k(t_i,t_j) &= \bC(t_i)\eu^{\bF (t_i-t_j)} \bP(t_j)\bC(t_j)^\top\\
  &= \bC(t_i)\eu^{\bF \Delta t_{i-1}}\cdots \eu^{\bF \Delta t_{j+1}} \eu^{\bF \Delta t_{j}}\bP(t_j)\bC(t_j)^\top\\
  &= \bC(t_i)\bT_{i-1:j+1} \bT_j\bP(t_j)\bC(t_j)^\top.
  \end{align*}
  Comparing with \eqref{eq:sss:symmetric}, we see that $\bK_{ij} = \bU_i\bT_{i-1:j+1}\bV_j^\top$ for $i>j$. The cases $i=j$ and $i<j$ follow similarly.
\end{proof}

\begin{remark}
An analogous result holds when the sequence $(t_1,\ldots,t_n)$ is nonincreasing. Indeed, if $\bK$ has the form \eqref{eq:sss:symmetric} and $\bJ_n$ is the exchange (reversal) matrix of order $n$, then the symmetric block permutation $(\bJ_n \otimes \bI_d)\bK(\bJ_n \otimes \bI_d)$ preserves the SSS structure: $(\bJ_n \otimes \bI_d)\bK(\bJ_n \otimes \bI_d)$ is again symmetric and SSS with the representation
\[ (\bJ_n \otimes \bI_d) \bK (\bJ_n \otimes \bI_d) = \hat{\bV}(\bI - \bN^p\hat{\bT}^\top)^{-1}\bN^p\hat{\bU}^\top + \hat{\bU}(\bN^\top)^p(\bI -\bN^p\hat{\bT}^\top)^{-\top}\hat{\bV}^\top
+  \hat{\bD}, \]
where $\hat{\bU}$, $\hat{\bT}$, $\hat{\bV}$, and $\hat{\bD}$ denote the block diagonal matrices obtained from $\bU$, $\bT$, $\bV$, and $\bD$ by reversing the order of the blocks.
\end{remark}

Computing the matrix exponential is inherently challenging \cite{MoV:03}, and obtaining accurate SSS generators requires particular care. If a Jordan decomposition $\bF = \bUps\bJ\bUps^{-1}$ is available, we readily obtain an alternative SSS representation of $\bK$ with generators
\begin{align}\label{eq:sss:generators:nonstationary:jordan}
	\bU_i = \bC(t_i)\bUps, \
	\bT_i = \eu^{\bJ \Delta t_{i}}, \
	\bV_i = \bC(t_i)\bP(t_i)\bUps^{-\top} \bT_i^{\top}, \
	\bD_i = \bC(t_i)\bP(t_i)\bC(t_i)^\top,
\end{align}
for $i \in \{1,\ldots,n\}$. We note that kernel matrix $\bK$ is also generator-representable semiseparable, corresponding to the SSS generators
\begin{align*}
	\bU_i = \bC(t_i)\eu^{\bF t_i}, \quad
	\bT_i = \bI, \quad
	\bV_i = \bC(t_i)\bP(t_i) \eu^{-\bF^\top t_i}, \quad
	\bD_i = \bC(t_i)\bP(t_i)\bC(t_i)^\top,
\end{align*}
for $i \in \{1,\ldots,n\}$. However, this representation comes with some numerical caveats: the generators may diverge, and the problem of evaluating the products $\bU_i\bV_j^\top$ may be severely ill-conditioned, which can lead to numerical issues in finite precision arithmetic.

We proceed in three steps, considering examples of stationary kernels, nonstationary kernels, and transformations that preserve the SSS structure.

\subsection{Stationary kernels}
For a stationary process $\bx(t)$, the covariance between $\bx(t)$ and $\bx(t')$ takes the form
\begin{align}\label{eq:sde:stationary:cov}
	\cov(\bx(t),\bx(t')) = \begin{cases}
		                       \eu^{\bF(t-t')}\bP_\infty,      & t\geq t', \\
		                       \bP_\infty\eu^{\bF^\top(t'-t)}, & t < t',
	                       \end{cases}
\end{align}
where $\bP_\infty$ is the solution to the Lyapunov equation
\begin{align}\label{eq:sde:stationary:lyapunov}
\bF\bP_\infty + \bP_\infty\bF^\top + \bG\bQ_{\bw}\bG^\top = \bzero.
\end{align}
An SSS representation of the covariance matrix $\bK$ is then readily obtained from \Cref{thm:sde-to-sss} by substituting $\bP(t) = \bP_\infty$ into \eqref{eq:sss:generators:nonstationary}.

We now consider two examples of stationary kernels that admit LTI SDE representations, namely the Matérn and celerite kernels, and show how to obtain SSS representations of the corresponding covariance matrices.

\subsubsection*{The Matérn kernel}
The Matérn kernel is a stationary kernel that is parameterized by a smoothness parameter $\nu$ and a length scale parameter $\ell$. For half-integer values of $\nu$, i.e., $\nu = m+1/2$ for $m\in\integers_+$, the Matérn kernel becomes the product of an exponential and a polynomial of degree \(m\) \cite{stein1999a,AbS:65}, i.e.,
\begin{align}\label{eq:kernel:matern}
k(t,t') = \sigma^2 \eu^{-\alpha |t-t'|} \frac{m!}{(2m)!}\sum_{k=0}^{m} \frac{(m+k)!}{k!(m-k)!}(2\alpha |t-t'|)^{m-k},
\end{align}
where $\sigma^2$ is a variance parameter and $\alpha = \sqrt{2\nu}/\ell$. In this case, the kernel can be derived from the SDE \cite{hartikainen2010a}
\begin{equation}\label{eq:state:matern}
	\frac{\dr\bx(t)}{\dr t}
	= \underbrace{
		\begin{bmatrix}
			0    & 1      &          &      & \\
			     & \ddots & \ddots   &        \\
			     &        & 0        & 1      \\
			-c_0 & \cdots & -c_{m-1} & -c_m
		\end{bmatrix}}_{\bF}\bx(t) +
	\be_{m+1}w(t),\quad
	y(t) = \be_{1}^\top \bx(t),
\end{equation}
where $w(t)$ is a white-noise process with diffusion coefficient
\begin{align*}
	q_w = \sigma^2 (2\alpha)^{2m+1} \binom{2m}{m}^{-1},
\end{align*}
with $\alpha = \sqrt{2\nu}/\ell$, and $c_0,\ldots,c_m$ are the coefficient of the characteristic polynomial
\begin{align*}
	\det(\lambda\bI - \bF)
	\journallayout{}{(\alpha+\lambda)^{m+1}}
	= \sum_{k=0}^{m+1} \binom{m+1}{k} \alpha^{m+1-k} \lambda^k
	= c_0 + c_1 \lambda + \cdots + c_m \lambda^m + \lambda^{m+1}.
\end{align*}
Under the assumption that the state $\bx(t)$ is stationary, the state covariance is given by \eqref{eq:sde:stationary:cov}. In this particular case, the Lyapunov equation \eqref{eq:sde:stationary:lyapunov} can be reduced to a system of $m+1$ linear equations in $m+1$ unknowns, as outlined in \cref{app:computing:Pinf}. \Cref{tbl:matern_sss} includes $\bP_\infty$ for $m \in \{0,1,2\}$ as well as a matrix $\bUps$ whose columns comprise a Jordan chain of $\bF$ at $-\alpha$. Given $\bP_\infty$ and $\bUps$, we can directly use the SSS representation of $\bK$ in \eqref{eq:sss:generators:nonstationary:jordan}.
\begin{table}[htbp]
	\centering
	\setlength{\tabcolsep}{12pt}
	\begin{tabular}{cccc}
		\toprule
		$m$ & $\alpha$                & $\bP_\infty$                                                        & $\bUps$                                            \\
		\midrule
		$0$ & $\frac{1}{\ell}$        & $\sigma^2$                                                          & $1$                                                \\
		$1$ & $\frac{\sqrt{3}}{\ell}$ & $\sigma^2\begin{bmatrix} 1 & 0 \\ 0 &\alpha^2 \end{bmatrix}$        & $\begin{bmatrix} 1 & 0 \\ -\alpha &1\end{bmatrix}$ \\[12pt]
		$2$ & $\frac{\sqrt{5}}{\ell}$ & $\sigma^2 \begin{bmatrix}
				                                1                   & 0                  & -\frac{\alpha^2}{3} \\
				                                0                   & \frac{\alpha^2}{3} & 0                   \\
				                                -\frac{\alpha^2}{3} & 0                  & \alpha^4
			                                \end{bmatrix}$ & $\begin{bmatrix}1 & 0 & 0 \\ -\alpha & 1 & 0 \\\alpha^2 & -2\alpha & 1\end{bmatrix}$                   \\ \addlinespace
		\bottomrule
	\end{tabular}\vspace*{4pt}
	\caption{Stationary covariance $\bP_\infty$ for the Matérn process with parameter $\nu = m + 1/2$. The Jordan form of $\bF$ is given by $\bJ=\bUps^{-1}\bF\bUps = -\alpha \bI + \bN^\top$.}
	\label{tbl:matern_sss}
\end{table}

\subsubsection*{The celerite kernel}
Another example of a stationary kernel is the celerite kernel, proposed as a model of stellar variations in astronomical time series in \cite{celerite}. The same formulation was proposed as a model of displacements in structural dynamics in \cite{CRT:21}, where it is referred to as the ``single degree of freedom'' (SDOF) kernel.
The celerite kernel can be expressed as
\begin{align}\label{eq:kernel:celerite}
    k(t,t') = \sigma^2 \eu^{-\alpha |t-t'|} \left(\cosh(\omega |t-t'|) + \frac{\alpha}{\omega}\sinh(\omega|t-t'|)\right),
\end{align}
where $\alpha = \omega_0/(2Q)$ and $\omega = \sqrt{\alpha^2-\omega_0^2}$ for some $\omega_0>0$ and $Q>0$. The kernel can be derived from a stochastically driven simple harmonic oscillator,
\begin{align*}
	\left[ \frac{\dr^2}{\dr t^2} + \frac{\omega_0}{Q} \frac{\dr}{\dr t} + \omega_0^2  \right] y(t)
	= w(t),
\end{align*}
where $\omega_0$ is the undamped angular frequency, $Q$ is the quality factor, and $w(t)$ is a Gaussian white noise process with diffusion coefficient $q_w = 2\sigma^2\omega_0^3/Q$. An equivalent state-space representation is given by
\begin{align}\label{eq:sde:celerite}
	\frac{\dr \bx(t)}{\dr t}
	= \underbrace{\begin{bmatrix} 0 & 1 \\ -\omega_0^2 & -\omega_0/Q \end{bmatrix}}_{\bF} \bx(t) + \begin{bmatrix} 0 \\ 1 \end{bmatrix} w(t), \quad
	y(t) = \be_1^\top \bx(t),
\end{align}
where the eigenvalues of $\bF$ are $\lambda_{\pm} = -\alpha \pm \omega$. The stationary covariance $\bP_\infty$ is obtained by solving the Lyapunov equation in \eqref{eq:sde:stationary:lyapunov}, which yields
\[ \bP_{\infty} = \sigma^2 \begin{bmatrix} 1 & 0 \\ 0 & \omega_0^2 \end{bmatrix}.\]
The special case where $Q=1/2$ corresponds to the critically damped case, where the eigenvalues of $\bF$ are equal and real, and $\bF$ is defective and has a Jordan decomposition $\bF = \bUps \bJ \bUps^{-1}$ with
\begin{align*}
	\bJ = \begin{bmatrix} -\omega_0 & 1 \\ 0 & -\omega_0 \end{bmatrix}, \qquad \bUps = \begin{bmatrix} 1 & 0 \\ -\omega_0 & 1 \end{bmatrix}.
\end{align*}
This is equivalent to the Matérn-$3/2$ process ($m=1$) with $\alpha = \sqrt{3}/\ell = \omega_0$.
For other values of $Q$, the eigenvalues are either real, distinct, and negative ($Q<1/2$) or complex conjugates ($Q>1/2$), and $\bF$ is diagonalizable:
\begin{align*}
	\bF = \bUps \bLambda \bUps^{-1}, \quad
	\bLambda = \begin{bmatrix} \lambda_+ & 0 \\ 0 & \lambda_- \end{bmatrix}, \quad
	\bUps = \begin{bmatrix} 1 & 1 \\ \lambda_+ & \lambda_- \end{bmatrix}.
\end{align*}
This eigendecomposition can be used to obtain an SSS representation of $\bK$ via \eqref{eq:sss:generators:nonstationary:jordan} with $\bJ = \bLambda$, but care must be taken to avoid numerical issues when $Q$ is close to $1/2$, in which case $\bUps$ is nearly singular. An alternative is to express $\eu^{\bF \tau}$ as a polynomial $\eu^{\bF \tau} = a(\tau) \bI + b(\tau) \bF$ (see \cite{BeS:93,HoJ:11}), which yields
\begin{align}\label{eq:sss:celerite:expFtau}
    \eu^{\bF\tau} = \eu^{-\alpha\tau} \left( \cosh(\omega\tau) \bI + \frac{\sinh(\omega \tau)}{\omega} (\bF + \alpha \bI)\right).
\end{align}
We note that the algorithms developed in \cite{celerite}, which are based on the generator representation of semiseparable matrices, are not applicable when $Q=1/2$. The SSS representation of $\bK$ does not have this limitation.

\subsection{Nonstationary kernels}
We now consider an example of nonstationary kernel that admit LTI SDE representations and show how to obtain SSS representations of the corresponding covariance matrices.

\subsubsection*{The spline kernel}
The spline kernel of order $m \in \naturals$ can be expressed as \cite{Wah:90}
\begin{align}\label{eq:kernel:spline}
	k(t,t') = \int_{0}^{\min(t,t')} \frac{(t-u)^{m-1}}{(m-1)!} \frac{(t'-u)^{m-1}}{(m-1)!} \, \dr u.
\end{align}
It can be derived from the $(m-1)$-times integrated Wiener process, which may be characterized by the $m$th-order stochastic differential equation
\[ \frac{\dr^m y(t)}{\dr t^m} = w(t) \]
where $w(t)$ is a Gaussian white noise process with diffusion coefficient $q_w = \sigma^2$. A state-space representation is given by \cite{pillonetto2010a,pillonetto2014a,pillonetto2022a,andersen2020a}
\begin{equation}\label{eq:state:iwp}
	\frac{\dr\bx(t)}{\dr t}   = \bF \bx(t) +
	\be_{1}w(t),
	\quad
	y(t) = \be_{m}^\top \bx(t),
\end{equation}
where $\bF=\bN$ is the lower shift matrix of order $m$. Taking $y(t)$ and its first $m-1$ derivatives to be zero at $t_0=0$, corresponding to the initial state $\bx(0) = \bzero$, leads to the state covariance
\begin{align*}
	\bP(t) = \cov(\bx(t),\bx(t)) =  q_w \int_{0}^{t} \eu^{\bF (t - s)}\be_{1}  \be_{1}^\top \eu^{\bF^\top (t - s)} \dr s, \quad t \geq 0.
\end{align*}
Noting that
\[\be_{1}^\top \eu^{\bF^\top \tau} = \begin{bmatrix} \varphi_0(\tau) & \varphi_1(\tau) &\cdots & \varphi_{m-1}(\tau) \end{bmatrix}, \quad \text{where}\ \varphi_i(\tau) = \tau^i/i!,\ \ i \in \integers_+,\]
and
\[ \int_{0}^t \varphi_i(t-s)\varphi_j(t-s)\, \dr s = \frac{1}{i!\, j!} \frac{t^{i+j+1}}{i+j+1}, \quad i,j \in \integers_+,\]
we can express $\bP(t)$ as a diagonally-scaled Hankel matrix,
\begin{align}
	\bP(t) = \sigma^2
	\bQ
	\begin{bmatrix}
		t               & \frac{t^2}{2}       & \cdots & \frac{t^{m}}{m}       \\
		\frac{t^2}{2}   & \frac{t^3}{3}       & \cdots & \frac{t^{m+1}}{m+1}   \\
		\vdots          & \vdots              & \ddots & \vdots                \\
		\frac{t^{m}}{m} & \frac{t^{m+1}}{m+1} & \cdots & \frac{t^{2m-1}}{2m-1}
	\end{bmatrix}
	\bQ, \quad
	\bQ = \diag(0!,1!,\ldots,(m-1)!)^{-1}.
\end{align}
Noting that $\eu^{\bF \tau} = \sum_{i=0}^{m-1} \varphi_i(\tau) \bN^i$, we readily obtain an SSS representation of $\bK$ from \eqref{eq:sss:generators:nonstationary}. \Cref{tbl:spline_sss} provides expressions for $\bP(t)$ and $\eu^{\bF \tau}$ for $m \in \{1,2,3\}$.
\begin{table}[htbp]
	\centering
	\setlength{\tabcolsep}{12pt}
	\begin{tabular}{ccc}
		\toprule
		$m$ & $\bP(t)$                                                                                                     & $\eu^{\bF \tau}$                                                               \\
		\midrule
		$1$ & $\sigma^2 t$                                                                                                 & $1$                                                                            \\
		$2$ & $\sigma^2 \begin{bmatrix}t & t^2/2 \\ t^2/2 & t^3/3\end{bmatrix}$                                            & $\begin{bmatrix} 1 & 0 \\ \tau & 1\end{bmatrix}$                               \\
		\addlinespace
		$3$ & $\sigma^2 \begin{bmatrix} t & t^2/2 & t^3/6 \\ t^2/2 & t^3/3 & t^4/8 \\ t^3/6 & t^4/8 & t^5/20\end{bmatrix}$ & $\begin{bmatrix} 1 & 0 & 0\\ \tau & 1 & 0 \\ \tau^2/2 & \tau & 1\end{bmatrix}$ \\
		\addlinespace
		\bottomrule
	\end{tabular}\vspace*{4pt}
	\caption{Covariance $\bP(t)$ and $\eu^{\bF \tau}$ for the $m$th-order spline kernel.}
	\label{tbl:spline_sss}
\end{table}

\subsection{Output modulation, time warping, and sums}\label{sub:sec:transformations}

The SSS structure of the covariance matrix $\bK$ is preserved under certain transformations of the kernel function $k(t,t')$. We will discuss three basic operations that preserve sequential semiseparability: output modulation, time warping, and kernel sums. These operations allow us to construct new kernels from existing ones while maintaining the SSS property of the corresponding covariance matrices.

\subsubsection*{Output modulation}
Let $\psi \colon \reals \to \reals$ be a given function, and define $\tilde\by(t) = \psi(t) \by(t)$. Then
\[ \tilde k(t,t') = \cov(\tilde \by(t), \tilde \by(t')) = \psi(t)\bC(t)\cov(\bx(t),\bx(t'))\bC(t')^\top \psi(t').\]
Thus, output modulation may be viewed as simply replacing $\bC(t)$ by $\tilde \bC(t) = \psi(t)\bC(t)$ in \eqref{eq:sde}, and hence the SSS structure is preserved.

As a simple but important example, consider a stationary Matérn-$1/2$ process
\[\frac{\dr x(t)}{\dr t} = -\alpha x(t) + w(t), \quad y(t)=x(t),\]
with diffusion coefficient $q_w = 2\alpha \sigma^2$. Applying the output modulation $\psi(t)=e^{-\beta t}$ with $\beta > 0$ gives
\[\tilde y(t)=e^{-\beta t}x(t),\]
and hence
\begin{align}\label{eq:kernel:dc}
\tilde k(t,t')=\sigma^2 e^{-\beta(t+t')}e^{-\alpha |t-t'|}.\end{align}
This kernel function is known as the diagonal-correlated (DC) kernel in system identification \cite{COL:12,pillonetto2014a}.
The tuned-correlated (TC) kernel is obtained as the special case where $\beta=\alpha$, which yields
\begin{align}\label{eq:kernel:tc}
\tilde k(t,t')=\sigma^2 e^{-2\alpha \max(t,t')}.
\end{align}
Thus, both the DC and TC kernels inherit the SSS property directly from the output-modulated Matérn-$1/2$ process.

\subsubsection*{Time warp}
Let $\phi \colon \reals \to \reals$ be a monotone function, and define $\check \by(t) = \by(\phi(t))$ with kernel function
\begin{align}
	\check{k}(t,t') = \cov(\by(\phi(t)),\by(\phi(t'))) = k(\phi(t),\phi(t')).
\end{align}
The requirement that $\phi$ be monotone ensures that the kernel matrix $\bK$ with blocks $\bK_{ij} = \check{k}(t_i,t_j)$ is SSS. Without monotonicity, there exists a permutation function $\pi$ such that $(\phi(t_{\pi(1)}),\ldots,\phi(t_{\pi(n)}))$ is nondecreasing, which implies that the kernel matrix with blocks $\check k(t_{\pi(i)},t_{\pi(j)})$ is SSS.

The stable-spline family provide an example of time warping. Let $k_m$ denote the
$m$th-order spline kernel. Applying the monotone time warp $\phi(t)=e^{-\alpha t}$ with $\alpha>0$ gives
\begin{align}\label{eq:kernel:ss}
\check k_m(t,t') = k_m(e^{-\alpha t},e^{-\alpha t'}).
\end{align}
Since $\phi$ is monotone, the resulting kernel matrix remains SSS after ordering the data points. The corresponding hyperparameter derivatives can also be represented in SSS form; we return to this point in \Cref{sub:sec:analytical:derivatives}. We note that the family of first-order stable spline kernels is equivalent to the family of TC kernels.
\begin{remark}
We note that while the time warp $\phi(t)=e^{-\alpha t}$ is monotone, it maps a nondecreasing sequence of time points $t_1,\ldots,t_n$ to a nonincreasing sequence of time points $\phi(t_1),\ldots,\phi(t_n)$. The resulting kernel matrix is still SSS, but the SSS generators are defined in reverse order; see as remarked following \Cref{thm:sde-to-sss}.
\end{remark}

\subsubsection*{Kernel sums}
A sum of kernels of the form
\[ k(t,t') = \sum_{i=1}^m k_i(t,t') \]
leads to a kernel matrix of the form $\bK = \sum_{i=1}^m \bK_i$, where $\bK_i$ is the kernel matrix generated by $k_i$. The matrix $\bK$ is SSS if $\bK_1,\ldots,\bK_m$ are SSS. We may view $k(t,t')$ as the covariance of the sum of $m$ independent processes, i.e.,
\[ \by(t) = \sum_{i=1}^m \by_i(t), \qquad k(t,t') = \cov(\by(t),\by(t')),  \]
where $\by_i(t)$ is described by an SDE
\begin{align*}
	\frac{\dr \bx_i(t)}{\dr t} = \bF_i \bx_i(t) + \bG_i \bw_i(t), \quad \by_i(t) =  \bC_i(t) \bx_i(t), \qquad i \in \{1,\ldots,m\}.
\end{align*}
The process $\by(t)$ can then be described by the SDE
\begin{align*}
\journallayout{\begin{aligned}}{}
	\frac{\dr \bx(t)}{\dr t} \journallayout{&}{}= \blkdiag(\bF_1,\ldots,\bF_m) \bx(t) + \blkdiag(\bG_1,\ldots,\bG_m)\bw(t),\journallayout{\\}{\quad} \by(t) \journallayout{&}{}=  \begin{bmatrix} \bC_1(t) & \cdots & \bC_m(t)\end{bmatrix} \bx(t),
\journallayout{\end{aligned}}{}
\end{align*}
where $\bx(t) = (\bx_1(t),\ldots,\bx_m(t))$ and $\bw(t) = (\bw_1(t),\ldots,\bw_m(t))$. We note that the $m$ state-space representations need not have the same state dimension, and the resulting SSS representation of $\bK$ has state dimension equal to the sum of the state dimensions of the individual SDEs. The SSS generators of $\bK$ can be obtained by concatenating the SSS generators of $\bK_1,\ldots,\bK_m$ in block-diagonal and block-row form, respectively.

\section{Linear-time SSS algorithms}\label{sec:algorithms}
We will now derive recursive algorithms for the primitives listed in \Cref{tab:algorithms:overview}, all of which have computational complexity linear in the number of data points. The recursions for the Cholesky and LDL$^\top$ factorizations, as well as triangular solves, are adaptations of existing SSS factorization results to the block notation adopted in this work. The main new algorithmic contribution is a recursion for $\tr(\bK^{-1}\bA)$ that avoids forming intermediate SSS products, making it directly applicable to the evaluation of analytic log-marginal-likelihood gradients. We also derive an algorithm for computing the diagonal elements of $\bK^{-1} = \bL^{-\top}\bL^{-1}$. This procedure can be interpreted as a modification of the SSS multiplication algorithm of \cite{chandrasekaran2003b}, specialized to compute and store only the diagonal blocks of the product. Several algorithms are presented in both Cholesky and LDL$^\top$ forms. The LDL$^\top$ formulation, in particular, admits a natural interpretation in terms of Kalman filtering, where $\bL^{-1}\by$ corresponds to the innovation process and $\btD$ to its associated covariance (\Cref{app:kalman:marginal}). We will frequently make use of the identity
\begin{align}\label{eq:chain_identity}
	\bT_{i:j} = \bT_{i:k+1}\bT_{k}\bT_{k-1:j}, \quad i\geq k \geq j,
\end{align}
which is a direct consequence of the definition of the descending-index chain product in \eqref{eq:chain:product}.
\begin{table}[H]
	\centering
	\caption{SSS algorithms and their role in Gaussian process regression. The matrix $\bK$ is assumed to be symmetric positive definite and SSS, and $\bA$ is assumed to be square and SSS.}
	\label{tab:algorithms:overview}
	\begin{tabular}{@{}lll@{}}
		\toprule
		Algorithm & Primitive & Location\\
		\midrule
		\ref{alg:cholesky}: Cholesky factorization                            & $\bK = \bL\bL^\top$
		& \Cref{sec:algorithms} \\
		\ref{alg:diagonal:inverse}: Diagonal blocks of inverse             & $(\bK^{-1})_{ii}$  \\
		\ref{alg:trace}: Trace of $\bK^{-1}\bA$                               & $\tr(\bK^{-1}\bA)$ \\
		\midrule
		\ref{alg:posterior:variance}: Diagonal blocks of Schur complement & $(\bK_{**}-\bK_*(\bK+\bSigma)^{-1}\bK_*^\top)_{ii}$
		& \Cref{sec:gp:regression}\\
		\midrule
		\ref{alg:forward:sub}: Forward substitution                           & $\bx = \bL^{-1}\bb$
		& \Cref{app:additional:algorithms} \\
		\ref{alg:backward:sub}: Backward substitution                         & $\bx = \bL^{-\top}\bb$ \\
		\ref{alg:explicit:inverse}: Inverse Cholesky factor                   & SSS generators for $\bL^{-1}$  \\
		\midrule
		\ref{alg:ldl}: $\text{LDL}^\top$ factorization                        & $\bK = \bL\bD\bL^\top$
		& \Cref{app:sec:ldl} \\
		\bottomrule
	\end{tabular}
\end{table}

\subsection{Cholesky factorization}\label{sub:sec:cholesky}
Consider a symmetric positive definite SSS matrix $\bK \in \reals^{nd \times nd}$ of the form \eqref{eq:sss:symmetric}, i.e.,
\begin{align*}
	\bK = \bU (\bI - \bN^p\bT)^{-1}\bN^p\bV^\top + \bV(\bN^p)^\top (\bI - \bN^p\bT)^{-\top}\bU^{\top} + \bD.
\end{align*}
We will show that the Cholesky factorization $\bK = \bL\bL^\top$ yields an SSS Cholesky factor of the form
\begin{equation}\label{eq:L}
	\bL = \bU (\bI - \bN^p\bT)^{-1}\bN^p\btV^\top + \btD,
\end{equation}
where $\btV$ and $\btD$ are block-diagonal matrices, i.e.,
\begin{equation}
	\bL_{ik} =
	\begin{cases}
		\bU_i\bT_{i-1:k+1}\btV_k^\top, & i > k,            \\
		\btD_i,                        & i = k,            \\
		\bzero,                        & \text{otherwise}.
	\end{cases}
\end{equation}
For $i > j$, the block $\bK_{ij} = (\bL\bL^\top)_{ij}$ is a sum of $j$ terms, i.e.,
\begin{align*}
	(\bL\bL^\top)_{ij}
	 & = \sum_{k=1}^{j-1}\bL_{ik}(\bL_{jk})^\top + \bL_{ij}\btD_{j}^\top                                                             \\
	 & = \sum_{k=1}^{j-1}\bU_i\bT_{i-1:k+1}\btV_k^\top\btV_k\bT_{j-1:k+1}^\top\bU_j^\top + \bU_i\bT_{i-1:j+1}\btV_j^\top\btD_j^\top.
\end{align*}
Using the identity \eqref{eq:chain_identity}, this simplifies to
\begin{align*}
	(\bL\bL^\top)_{ij} & = \bU_i\bT_{i-1:j+1}\left(\bT_j\left(\sum_{k=1}^{j-1}\bT_{j-1:k+1}\btV_k^\top\btV_k\bT_{j-1:k+1}^\top\right)\bU_j^\top + \btV_j^\top\btD_j^\top\right) \\
	 & = \bU_i\bT_{i-1:j+1}\left(\bT_j\bm{\Psi}_j\bU_j^\top + \btV_j^\top\btD_j^\top\right),
\end{align*}
where $\bm{\Psi}_j = \sum_{k=1}^{j-1}\bT_{j-1:k+1}\btV_k^\top\btV_k\bT_{j-1:k+1}^\top$ is defined recursively by
\begin{align}\label{eq:cholesky:psi:recursion}
	\bm{\Psi}_1 = \bzero, \qquad \bm{\Psi}_{i+1} = \bT_{i}\bm{\Psi}_{i}\bT_{i}^\top + \btV_{i}^\top\btV_{i},  \quad i = 1,\ldots,n-1.
\end{align}
Matching $(\bL\bL^\top)_{ij}$ with $\bK_{ij} = \bU_i\bT_{i-1:j+1}\bV_j^\top$ yields $\bV_j^\top = \bT_j\bm{\Psi}_j\bU_j^\top + \btV_j^\top\btD_j^\top$, which implies that
\begin{align}\label{eq:offdiag}
	\btV_j^\top = (\bV_j^\top  -\bT_j\bm{\Psi}_j\bU_j^\top)\btD_j^{-\top}.
\end{align}
For the diagonal blocks $\bK_{ii} = (\bL\bL^\top)_{ii}$, we have
\begin{align*}
	(\bL\bL^\top)_{ii}
	= \sum_{k=1}^{i-1}\bL_{ik}(\bL_{ik})^\top + \btD_i\btD_{i}^\top
	= \bU_i\bm{\Psi}_i\bU_i^\top + \btD_i\btD_i^\top,
\end{align*}
and matching $(\bL\bL^\top)_{ii}$ with $\bK_{ii}=\bD_i$ yields $\bU_i\bm{\Psi}_i\bU_i^\top + \btD_i\btD_i^\top = \bD_i$, which can be rearranged as
\begin{equation}\label{eq:diag}
	\btD_i\btD_i^\top = \bD_i - \bU_i\bm{\Psi}_i\bU_i^\top,
\end{equation}
where we take $\btD_i$ to be the lower triangular Cholesky factor of the right-hand side.
The relations \eqref{eq:cholesky:psi:recursion}, \eqref{eq:offdiag}, and \eqref{eq:diag} form the basis of the SSS Cholesky factorization presented in \Cref{alg:cholesky}, which has computational complexity $\cO(n(d^3+p^3))$.

\begin{algorithm}[H]
	\caption{Cholesky factorization of symmetric positive definite SSS matrix $\bK$}
	\label{alg:cholesky}
	\KwIn{SSS matrix $\bK$ of the form \eqref{eq:sss:symmetric}, defined by block-diagonal matrices $(\bU,\bT,\bV,\bD)$.}
	\KwOut{Block-diagonal matrices $\btV$ and $\btD$ such that $\bL=\bU(\bI - \bN^p\bT)^{-1}\bN^p\btV^\top + \btD$ and $\bK = \bL\bL^\top$.}
	$\bm{\Psi} \gets \bzero_{p \times p}$\;
	\For{$i = 1$ \KwTo $n$}{
		$\bM \gets \bT_i\bm{\Psi}$\;
		$\btV_i^\top \gets \bV_i^\top - \bM\bU_i^\top$\;
		$\btD_i \gets \chol(\bD_i - \bU_i\bm{\Psi}\bU_i^\top)$ \tcp*{Lower triangular Cholesky factor}
		$\btV_i^\top \gets \btV_i^\top\btD_i^{-\top}$\;
		$\bm{\Psi} \gets \bM\bT_i^\top + \btV_i^\top\btV_i$\;
	}
	\Return{$(\btV_1,\ldots,\btV_n)$, $(\btD_1,\ldots,\btD_n)$}
\end{algorithm}
\begin{remark}
	An SSS representation of $\bL^{-1}$ can be computed in linear time using \Cref{alg:explicit:inverse}; the derivation and algorithm are provided in \Cref{app:additional:algorithms}. When $\bC_1,\ldots,\bC_n$ are square and nonsingular, the inverse of $\bL$ is block bidiagonal, and $\bK^{-1}$ is therefore block tridiagonal. This is shown in \Cref{app:block:tridiagonal}.
\end{remark}

\subsection{Diagonal blocks of $\bK^{-1}$}\label{sub:sec:diagonal:inverse}

The diagonal blocks of $\bK^{-1}$ can, in principle, be obtained directly from the product $\bL^{-\top}\bL^{-1}$. In fact, the algorithm that we present may be viewed as a special case of the general multiplication of two SSS matrices algorithm from \cite{chandrasekaran2003b}. Given an SSS representation of $\bL^{-1}$ of the form
\begin{align}\label{eq:sss:inverse:cholesky}
	\bL^{-1} = \boU(\bI - \bN\boT)^{-1}\bN\boV^\top + \boD,
\end{align}
which can be computed from the SSS representation of $\bL$ using \Cref{alg:explicit:inverse}, we can express the diagonal blocks of $\bK^{-1}$ as
\begin{align}
	\notag
	(\bK^{-1})_{ii}
	 & = \sum_{k=i+1}^{n}(\bL^{-1})_{ki}^\top(\bL^{-1})_{ki} + \boD_i^\top\boD_i                                 \\
	\notag
	 & = \sum_{k=i+1}^{n}\boV_i\boT_{k-1:i+1}^\top\boU_k^\top\boU_k\boT_{k-1:i+1}\boV_i^\top + \boD_i^\top\boD_i \\
	\label{eq:inverse:diagonal}
	 & = \boV_i\bm{\Phi}_i\boV_i^\top + \boD_i^\top\boD_i,
\end{align}
where $\bm{\Phi}_i = \sum_{k=i+1}^{n}\boT_{k-1:i+1}^\top\boU_k^\top\boU_k\boT_{k-1:i+1}$ can be computed recursively as
\begin{align}\label{eq:inverse:diagonal:recursion}
	\bm{\Phi}_n = \bzero, \qquad \bm{\Phi}_{i-1} = \boT_i^\top\bm{\Phi}_i\boT_i + \boU_i^\top\boU_i, \quad i = n,\ldots,2.
\end{align}
The resulting procedure for computing the diagonal blocks of $\bK^{-1}$ is presented in \Cref{alg:diagonal:inverse}, which has computational complexity $\cO(n(d^3+p^3))$.

\begin{algorithm}[H]
	\caption{Diagonal blocks of inverse of symmetric positive definite SSS matrix $\bK$}
	\label{alg:diagonal:inverse}
	\KwIn{Inverse Cholesky factor $\bL^{-1}$ of $\bK$, defined by block-diagonal matrices $(\boU,\boT,\boV,\boD)$.}
	\KwOut{$\bm{\Omega}_i = (\bK^{-1})_{ii}$ for $i\in\{1, \ldots, n\}$.}
	$\bm{\Phi} \gets \bzero_{p \times p}$\;
	\For{$i = n$ \KwTo $1$}{
		$\bm{\Omega}_i   \gets \boV_i\bm{\Phi}\boV_i^\top + \boD_i^\top\boD_i$\;
		$\bm{\Phi}       \gets \boT_i^\top\bm{\Phi}\boT_i + \boU_i^\top\boU_i$\;
	}
	\Return{$(\bm{\Omega}_1,\ldots,\bm{\Omega}_n)$}
\end{algorithm}

\subsection{Trace of $\bK^{-1}\bA$}\label{sub:sec:trace}

We now derive a novel algorithm for computing the trace of $\bK^{-1}\bA$, where $\bK$ is once again assumed to be a symmetric positive definite SSS matrix of the form \eqref{eq:sss:symmetric}, and $\bA$ is a general nonsymmetric SSS matrix of the form \eqref{eq:def:nonsym:A}, i.e.,
\begin{align*}
	\bA = \bbU (\bI - \bN^p\bbT)^{-1}\bN^p\bbV^\top + \bbP(\bN^q)^\top (\bI - \bN^q\bbR)^{-\top}\bbQ^\top + \bbD,
\end{align*}
where we use the accented letters $\bbT$, $\bbU$, $\bbV$, $\bbP$, $\bbQ$, $\bbR$, and $\bbD$ for the SSS generators of $\bA$ to distinguish them from the SSS generators of $\bK$. Furthermore, we will assume that the inverse Cholesky factor of $\bK$ has an SSS representation of the form \eqref{eq:sss:inverse:cholesky}, i.e.,
\[\bL^{-1} = \boU(\bI - \bN^p\boT)^{-1}\bN^p\boV^\top + \boD.\]
Using the cyclic property of the trace, we then have
\begin{align*}
	\tr(\bK^{-1}\bA) = \tr(\bL^{-1}\bA\bL^{-\top}) = \sum_{i=1}^{n}\tr(\bm{\Theta}_i),
\end{align*}
where $\bm{\Theta}_i = (\bL^{-1}\bA\bL^{-\top})_{ii}$ denotes the $i$th diagonal block of $\bL^{-1}\bA\bL^{-\top}$. We note that $\bm{\Theta}_i$ is determined entirely by the $i$th block row of $\bL^{-1}$ and the $id \times id$ leading principal submatrices of $\bA$, which we will denote $\bA_{(i)}$, i.e., $\bm{\Theta}_1 = \boD_1 \bbD_1\boD_1^{\top}$ and for $i \in \{2,\ldots,n\}$, we have that $\bm{\Theta}_i$ is given by
\journallayout{\begingroup\setlength{\arraycolsep}{3pt}}{} %
\begin{align*}
\underbrace{	\begin{bmatrix}
						\boV_1 \boT_{i-1:2}^\top \boU_i^\top \\
						\boV_2 \boT_{i-1:3}^\top \boU_i^\top \\
						\vdots                               \\
						\boD_{i}^{\top}
					\end{bmatrix}^\top}_{i \text{th block row of } \bL^{-1}}
\underbrace{
	\begin{bmatrix}
		\bbD_1                        & \bbP_1\bbQ_2^\top             & \cdots & \bbP_1\bbR_{i-1:2}^\top\bbQ_i^\top \\
		\bbU_2\bbV_1^\top             & \bbD_2                        & \cdots & \bbP_2\bbR_{i-1:3}^\top\bbQ_i^\top \\
		\vdots                        & \vdots                        & \ddots & \vdots                             \\
		\bbU_i\bbT_{i-1:2}\bbV_1^\top & \bbU_i\bbT_{i-1:3}\bbV_2^\top & \cdots & \bbD_i
	\end{bmatrix}}_{\bA_{(i)}}
\begin{bmatrix}
	\boV_1 \boT_{i-1:2}^\top \boU_i^\top \\
	\boV_2 \boT_{i-1:3}^\top \boU_i^\top \\
	\vdots                               \\
	\boD_{i}^{\top}
\end{bmatrix}.
\end{align*}
\journallayout{\endgroup}{}
To simplify notation, we define
\begin{align*}
	\bm{\Upsilon}_i                =
	\begin{bmatrix}
		\boV_1 \boT_{i-1:2}^\top \\
		\boV_2 \boT_{i-1:3}^\top \\
		\vdots                   \\
		\boV_{i-1}
	\end{bmatrix},  \journallayout{\quad}{\qquad}
	\bm{\Xi}_i                     =
	\begin{bmatrix}
		\bbV_1  \bbT_{i-1:2}^\top \\
		\bbV_2  \bbT_{i-1:3}^\top \\
		\vdots                    \\
		\bbV_{i-1}
	\end{bmatrix}, \journallayout{\quad}{\qquad}
	\bm{\Gamma}_i                  =
	\begin{bmatrix}
		\bbP_1\bbR_{i-1:2}^\top \\
		\bbP_2\bbR_{i-1:3}^\top \\
		\vdots                  \\
		\bbP_{i-1}
	\end{bmatrix}, \journallayout{\quad}{\qquad} i \in \{2,\ldots,n\},
\end{align*}
such that for $i \in \{2,\ldots,n\}$, we can write
\begin{align*}
	\bm{\Theta}_i & = \begin{bmatrix}\bm{\Upsilon}_i\boU_i^\top \\ \boD_i^{\top}\end{bmatrix}^\top
	\begin{bmatrix}\bA_{(i-1)} & \bm{\Gamma}_i\bbQ_i^\top \\ \bbU_i\bm{\Xi}_i^\top & \bbD_i\end{bmatrix}
	\begin{bmatrix}\bm{\Upsilon}_i\boU_i^\top \\ \boD_i^{\top}\end{bmatrix}                                                                                                                                                                                                                                                                                  \\
	              & = \underbrace{\boD_i\bbD_i\boD_i^{\top}}_{\text{(I)}} + \underbrace{\boD_i\bbU_i\bm{\Xi}_i^\top \bm{\Upsilon}_i \boU_i^\top}_{\text{(II)}} + \underbrace{\boU_i \bm{\Upsilon}_i^\top\bm{\Gamma}_i\bbQ_i^\top \boD_i^{\top}}_{\text{(III)}} + \underbrace{\boU_i\bm{\Upsilon}_i^\top\bA_{(i-1)}\bm{\Upsilon}_i\boU_i^\top}_{\text{(IV)}}.
\end{align*}
The first term can be computed directly, whereas the remaining three terms can be evaluated recursively. We now demonstrate this, beginning with the observation that $\bm{\Upsilon}_2 = \boV_1$, $\bm{\Xi}_2 = \bbV_1$, and $\bm{\Gamma}_2 = \bbP_1$, and
\begin{align*}
	\bm{\Upsilon}_{i+1} = \begin{bmatrix}\bm{\Upsilon}_i\boT_i^\top \\ \boV_i\end{bmatrix}, \journallayout{\quad}{\qquad}
	\bm{\Xi}_{i+1} = \begin{bmatrix}\bm{\Xi}_i\bbT_i^\top \\ \bbV_i\end{bmatrix},			\journallayout{\quad}{\qquad}
	\bm{\Gamma}_{i+1} = \begin{bmatrix}\bm{\Gamma}_i\bbR_i^\top \\ \bbP_i\end{bmatrix}, 	\journallayout{\quad}{\qquad} i \in \{2,\ldots,n-1\}.
\end{align*}
Using these recurrences, and defining $\bW_i = \bm{\Xi}_i^\top\bm{\Upsilon}_i$ and $\bZ_i = \bm{\Upsilon}_i^\top\bm{\Gamma}_i$, we obtain the recurrences
\begin{align}\label{eq:trace:recurrences}
	\bW_{i+1} = \bbT_i \bW_i\boT_i^\top + \bbV_i^\top\boV_i, \qquad
	\bZ_{i+1} = \boT_i\bZ_i\bbR_i^\top + \boV_i^\top\bbP_i,
\end{align}
for $i \in \{1,\ldots,n-1\}$, with $\bW_1 = \bzero$ and $\bZ_1=\bzero$. Similarly, in term (IV), the subexpression $\bm{\Pi}_i = \bm{\Upsilon}_i^\top\bA_{(i-1)}\bm{\Upsilon}_i$ satisfies the recurrence relation
\begin{equation}\label{eq:trace:recurrence:IV}
	\begin{aligned}
		\bm{\Pi}_{i+1} & = \begin{bmatrix}\bm{\Upsilon}_i\boT_i^\top \\ \boV_i\end{bmatrix}^\top
		\begin{bmatrix}\bA_{(i-1)} & \bm{\Gamma}_i\bbQ_i^\top \\ \bbU_i\bm{\Xi}_i^\top & \bbD_i\end{bmatrix}
		\begin{bmatrix}\bm{\Upsilon}_i\boT_i^\top \\ \boV_i\end{bmatrix}                                                                               \\
		               & =\boT_i \bm{\Pi}_i\boT_i^\top + \boV_i^\top\bbU_i\bW_i\boT_i^\top + \boT_i \bZ_i\bbQ_i^\top\boV_i  + \boV_i^\top\bbD_i\boV_i,
	\end{aligned}
\end{equation}
for $i \in \{1,\ldots,n-1\}$ where $\bm{\Pi}_1 = \bzero$. \Cref{alg:trace} computes $\tr(\bK^{-1}\bA)$ using the recurrence relations \eqref{eq:trace:recurrences} and \eqref{eq:trace:recurrence:IV}. We note that \Cref{alg:trace} computes the diagonal blocks of $\bL^{-1}\bA \bL^{-\top}$, which can be returned instead of the trace if desired. In either case, the algorithm has computational complexity $\cO(n(d^3+p^3))$.

\begin{algorithm}[H]
	\caption{Trace of $\bK^{-1}\bA$ where $\bK$ and $\bA$ are SSS matrices, and $\bK$ is symmetric positive definite}
	\label{alg:trace}
	\KwIn{SSS matrix $\bA$ of the form \eqref{eq:def:nonsym:A}, defined by $(\bbU, \bbT, \bbV, \bbD, \bbP, \bbR, \bbQ)$, and inverse Cholesky factor of $\bK$ of the form \eqref{eq:sss:inverse:cholesky}, defined by $(\boU, \boT, \boV, \boD)$.}
	\KwOut{$\tr(\bK^{-1}\bA)$}
	$\bm{\Pi} \gets \bzero$, $\bZ \gets \bzero$, $\bW \gets \bzero$\;
	\For{$i = 1$ \KwTo $n$}{
		$\bm{\Theta}_i \gets \boD_i\bbD_i\boD_i^\top + \boD_i\bbU_i\bW\boU_i^\top + \boU_i\bZ\bbQ_i^\top\boD_i^\top  + \boU_i\bm{\Pi}\boU_i^\top$\;
		$\bm{\Pi} \gets \boT_i\bm{\Pi}\boT_i^\top + \boV_i^\top \bbU_i\bW\boT_i^\top +  \boT_i\bZ\bbQ_i^\top\boV_i+ \boV_i^\top\bbD_i\boV_i$\;
		$\bW \gets \bbT_i\bW\boT_i^\top + \bbV_i^\top\boV_i$\;
		$\bZ \gets \boT_i\bZ\bbR_i^\top + \boV_i^\top\bbP_i$\;
	}
	\Return{$\sum_{i=1}^n \tr(\bm{\Theta}_i)$}
\end{algorithm}

\section{Gaussian process regression}\label{sec:gp:regression}

The algorithms presented in \Cref{sec:algorithms} and \Cref{app:additional:algorithms} can be used to accelerate GP regression with kernels that admit a sequentially semiseparable representation. \Cref{tab:algorithms:overview} \journallayout{summarizes these algorithms.}{summarizes these algorithms and their roles as GP primitives.} We will assume that $f \colon \reals \to \reals^d$ is a vector-valued function that we wish to learn from training data $\cD = \{(t_i, \by_i)\}_{i=1}^n$ where $\by_i$ are observations of $\bff_i = f(t_i)$, generated according to the model
\begin{equation}
	\by_i = \bff_i + \bm{\epsilon}_i, \qquad \bm{\epsilon}_i \sim \cN(\bzero, \bm{\Sigma}_i), \qquad i \in \{1,\ldots,n \},
\end{equation}
where we assume that $\bm{\epsilon}_1,\ldots,\bm{\epsilon}_n$ are independent.

The GP regression model makes the a priori assumption that the joint distribution of $\bff_1,\ldots,\bff_n$ is Gaussian with mean zero and covariance $\bK$, which is a symmetric matrix of order $nd$ with $d \times d$ blocks given by
\[ \bK_{ij} = k(t_i,t_j; \btheta), \qquad i,j \in \{1,\ldots,n\}, \]
where $k \colon \reals \times \reals \to \reals^{d\times d}$ is a positive definite kernel function parameterized by a vector of hyperparameters $\btheta\in \reals^{\nhyper }$. Here, we limit our attention to kernels that admit a sequentially semiseparable representation, which implies that $\bK$ is a positive semidefinite SSS matrix.

Given a nondecreasing sequence of $\ntest$ test inputs $\bt_* = (t_{*1},\ldots,t_{*\ntest})$, the joint prior distribution of $\by = (\by_1, \dots, \by_n)$ and $\bff_* = (f(t_{*1}),\ldots,f(t_{*\ntest}))$ is also Gaussian, i.e.,
\begin{equation}\label{eq:joint:distribution}
	\begin{bmatrix}
		\by \\ \bff_*
	\end{bmatrix}
	\sim
	\cN\left(
	\bzero,
	\begin{bmatrix}
		\bK + \bm{\Sigma} & \bK_*^\top \\
		\bK_*             & \bK_{**}
	\end{bmatrix}
	\right),
\end{equation}
where $\bm{\Sigma} = \blkdiag(\bm{\Sigma}_1,\ldots,\bm{\Sigma}_n)$ and where the $(i,j)$th blocks of $\bK$, $\bK_{*}$, and $\bK_{**}$ are given by
\begin{align*}
	\bK_{ij}=k(t_i,t_j;\btheta), \quad
	                             [\bK_{*}]_{ij}=k(t_{*i},t_j;\btheta), \quad
	                                                                   [\bK_{**}]_{ij}=k(t_{*i},t_{*j};\btheta).
\end{align*}
The posterior distribution of $\bff_*$ given $\by$, which is also commonly known as the posterior predictive distribution, is given by
\begin{align}\label{eq:gp:posterior}
	\bff_*\,|\,\by \sim \cN(\bm{\mu}_{\bff_*|\by}, \bm{\Sigma}_{\bff_*|\by}),
\end{align}
where $\bm{\mu}_{\bff_*|\by} = \bK_* (\bK + \bm{\Sigma})^{-1} \by$ and $\bm{\Sigma}_{\bff_*|\by} = \bK_{**} - \bK_* (\bK + \bm{\Sigma})^{-1} \bK_*^\top$.

To fit the GP model to the training data, we need to estimate the hyperparameters $\btheta$ of the kernel function.
A common approach is to maximize the log marginal likelihood, $\log p(\by\,|\,\btheta)$, which can be expressed as the optimization problem
\begin{equation}\label{eq:gp:loglik}
	\hat \btheta= \argmax_{\btheta\in \Theta} \,\left( -\tfrac{1}{2}\by^\top (\bK + \bm{\Sigma})^{-1} \by - \tfrac{1}{2}\log\det(\bK + \bm{\Sigma}) - \tfrac{nd}{2}\log(2\pi) \right),
\end{equation}
where $\bTheta\subset \reals^{\nhyper }$ denotes the set of admissible hyperparameters. A local gradient-based optimization method, such as L-BFGS-B \cite{BLNZ:95}, is commonly used when $\bTheta$ is defined by simple lower and upper bounds on the hyperparameters. The gradient evaluation amounts to computing the partial derivatives of the log marginal likelihood with respect to each hyperparameter $\theta_k$, which can be computed analytically as
\begin{equation}\label{eq:gp:grad}
\journallayout{\begin{aligned}}{}
	\frac{\partial}{\partial\theta_k}\log p(\by\,|\,\btheta) \journallayout{&}{}= \tfrac{1}{2}\by^\top (\bK+\bm{\Sigma})^{-1}\frac{\partial (\bK+\bm{\Sigma})}{\partial\theta_k}(\bK+\bm{\Sigma})^{-1}\by \journallayout{\\ &\quad}{}- \tfrac{1}{2}\tr\!\left((\bK+\bm{\Sigma})^{-1}\frac{\partial(\bK+\bm{\Sigma})}{\partial\theta_k}\right).
\journallayout{\end{aligned}}{}
\end{equation}
For a general kernel, the computational cost of forming and factorizing $\bK+\bm{\Sigma}$ scales as $\cO(n^3 d^3)$, with a storage requirement of $\cO(n^2d^2)$. Once the Cholesky factorization $\bK+\bm{\Sigma} = \bL\bL^\top$ has been computed for a given $\btheta$, evaluation of the posterior mean and the diagonal of the posterior covariance scales as $\cO(n^2d^2+\ntest n d^2 )$ and $\cO(\ntest n^2 d^3)$, respectively. Similarly, given $\bL$, the evaluation of the log marginal likelihood scales as $\cO(n^2d^2)$ and evaluation of its gradient scales as $\cO(\nhyper n^3 d^3)$. In contrast, when $\bK+\bm{\Sigma}$ admits an SSS representation, storage and computational cost scale linearly with the number $n$ of input locations for fixed block size $d$ and SSS orders. Given an SSS representation of $\bK + \bm{\Sigma}$, the first step is to compute the Cholesky factorization $\bK+\bm{\Sigma}=\bL\bL^\top$, which can be computed using \Cref{alg:cholesky}. This, in turn, makes it possible to evaluate the terms $(\bK+\bm{\Sigma})^{-1} \by$ and $\log\det(\bK+\bm{\Sigma})$ in linear time for $d$ and $p$ fixed.
In the remainder of this section, we focus on the log marginal likelihood and describe how to efficiently compute its partial derivatives and how to efficiently evaluate and draw samples from the posterior distribution.

\begin{remark}
    An alternative to maximum likelihood estimation for hyperparameter selection is generalized cross validation (GCV) \cite{GHW:79}, which does not require a probabilistic noise model. As shown in \Cref{app:gcv}, the GCV score and its gradient can likewise be evaluated in linear time using the SSS representation of $\bK+\bm{\Sigma}$.
\end{remark}

\subsection{Partial derivatives of the log marginal likelihood}\label{sub:sec:analytical:derivatives}

In practice, the gradient of the log marginal likelihood is often approximated numerically using finite differences.
Forward finite differences require $\nhyper+1$ likelihood evaluations, while central differences require $2\nhyper +1$, and hence the computational cost of this gradient approximation scales as $\cO(\nhyper n(d^3 +p^3))$ when the SSS structure is exploited. Another approach is to use reverse-mode derivative recursions \cite{celerite,dfm2018} or, as we describe next, to compute the partial derivatives of the log marginal likelihood analytically by exploiting the SSS structure of $\bK+\bm{\Sigma}$. Specifically, we will consider three cases for how the hyperparameter $\theta_k$ enters the GP model:
\begin{enumerate}
	\item Only the output model $\Mo(t; \btheta)$ depends on $\theta_k$.
	\item Only the state $\bx(t;\btheta)$ depends on $\theta_k$.
	\item Time warping: $\check{\by}(t) = \by(\phi(t;\theta_k))$, where $\phi(t;\theta_k)$ is a time warp function.
\end{enumerate}
In case 1, we will show that the partial derivative of $\bK$ with respect to $\theta_k$ can be expressed as
\begin{align}\label{eq:gp:partial_deriv:general}
	\frac{\partial \bK}{\partial \theta_k} = \bA + \bA^\top,
\end{align}
where $\bA$ is a nonsymmetric SSS matrix. This leads to the following expression for the partial derivative of the log marginal likelihood with respect to $\theta_k$:
\begin{align*}
	\frac{\partial }{\partial \theta_k} \log p(\by\,|\, \btheta)
	 & = \by^\top (\bK+\bm{\Sigma})^{-1} \bA  (\bK+\bm{\Sigma})^{-1}\by - \tr\left((\bK+\bm{\Sigma})^{-1}\bA\right)                                \\
	 & \qquad +\tfrac{1}{2}\by^\top (\bK+\bm{\Sigma})^{-1} \frac{\partial \bm{\Sigma}}{\partial \theta_k}(\bK+\bm{\Sigma})^{-1}\by  - \tfrac{1}{2}
	\tr\left((\bK+\bm{\Sigma})^{-1}\frac{\partial \bm{\Sigma}}{\partial \theta_k} \right).
\end{align*}
Since matrix-vector products with $\bA$ are straightforward to compute and the term $\tr((\bK + \bm{\Sigma})^{-1}\bA)$ can be evaluated using \Cref{alg:trace}, the first two terms can be computed in linear time. Similarly, using \Cref{alg:diagonal:inverse} to compute the diagonal blocks of $(\bK+\bm{\Sigma})^{-1}$ and noting that $\partial \bm{\Sigma}/\partial \theta_k$ is block diagonal, the last two terms can also be computed in linear time. The full gradient of the log marginal likelihood can therefore be computed in $\cO(\nhyper  n (d^3 + p^3))$ time, which matches the cost of finite difference approximations.

\subsubsection*{Case 1: only the output model depends on the hyperparameter}
We first consider the case where only the output model $\Mo(t;\btheta)$ depends on the hyperparameter $\theta_k$. This is the case, for example, when the hyperparameter $\theta_k$ is an output modulation parameter, such as in the DC kernel; see \Cref{sub:sec:transformations}. Applying the product rule, we have
\begin{align*}
	\frac{\partial k(t,t')}{\partial \theta_k} =
	\begin{cases}
		\frac{\partial \bC(t;\btheta)}{\partial \theta_k}\eu^{\bF (t-t')} \bP(t')\bC(t;\btheta)^\top + \bC(t;\btheta)\eu^{\bF (t-t')} \bP(t')\frac{\partial\bC(t';\btheta)^\top}{\partial\theta_k },    & t \geq t', \\
		\frac{\partial\bC(t;\btheta)}{\partial\theta_k}\bP(t)\eu^{\bF^\top (t'-t)}\bC(t';\btheta)^\top + \bC(t;\btheta)\bP(t)\eu^{\bF^\top (t'-t)}\frac{\partial\bC(t';\btheta)^\top}{\partial\theta_k }, & t < t'.
	\end{cases}
\end{align*}
Evaluating this expression at training-input pairs $(t_i, t_j)$ yields a decomposition of the form \eqref{eq:gp:partial_deriv:general} with the SSS representation of $\bA$ being
\begin{align*}
	 & \bU_i = \frac{\partial \Mo(t_i;\btheta)}{\partial \theta_k}, \quad
	\bT_i = \eu^{\bF\Delta t_i}, \quad
	\bV_i = \Mo(t_i;\btheta)\bP(t_i)\bT_i^\top, \quad
	\bP_i = \frac{\partial \Mo(t_i;\btheta)}{\partial \theta_k}\bP(t_i)\bT_i^\top, \\
	 & \bQ_i = \Mo(t_i;\btheta), \quad
	\bR_i = \eu^{\bF\Delta t_i}, \quad
	\bD_i = \frac{\partial \Mo(t_i;\btheta)}{\partial \theta_k}\bP(t_i)\Mo(t_i;\btheta)^\top, \quad
	i \in \{1,\ldots,n\}.
\end{align*}

\subsubsection*{Case 2: only the state depends on the hyperparameter}
Next, we consider the case where only the state $\bx(t;\btheta)$ depends on the hyperparameter $\theta_k$. This is the case, for example, when the hyperparameter $\theta_k$ is a length scale in a Matérn kernel or an undamped angular frequency or quality factor in a celerite kernel. For a nonstationary kernel $k$ of the form \eqref{eq:sde:nonstationary:cov}, differentiating $k$ directly yields
\begin{align}\label{eq:kernel:deriv:state:nonstationary}
	\frac{\partial k(t,t')}{\partial \theta_k}
	 & =
	\begin{cases}
		\bC(t)\left(\frac{\partial \eu^{\bF(\btheta) (t-t')}}{\partial\theta_k} \bP(t';\btheta) + \eu^{\bF(\btheta) (t-t')}\frac{\partial \bP(t';\btheta)}{\partial\theta_k}\right)\bC(t')^\top,        & t \geq t', \\
		\bC(t)\left(\frac{\partial \bP(t;\btheta)}{\partial\theta_k}\eu^{\bF(\btheta)^\top (t'-t)} + \bP(t;\btheta)\frac{\partial \eu^{\bF(\btheta)^\top (t'-t)}}{\partial\theta_k}\right)\bC(t')^\top, & t < t'.
	\end{cases}
\end{align}
For a stationary kernel, we simply replace $\bP(t;\btheta)$ with $\bP_\infty$. The partial derivative of the matrix exponential $\eu^{\bF(\btheta) (t-t')}$ with respect to $\theta_k$ can be computed using the Fréchet derivative formula for the matrix exponential \cite{NaH:95,Hig:08}, i.e., we have
\begin{align}
	\eu^{\tilde{\bF}(\btheta) \tau} =
	\begin{bmatrix}
		\eu^{\bF(\btheta)\tau}                                     & \bzero                 \\
		\frac{\partial \eu^{\bF(\btheta) \tau}}{\partial \theta_k} & \eu^{\bF(\btheta)\tau}\end{bmatrix},
	\quad \text{where $\tau = t-t'$ and }
	\tilde{\bF}(\btheta) = \begin{bmatrix}\bF(\btheta) & \bzero \\ \frac{\partial \bF(\btheta)}{\partial \theta_k} & \bF(\btheta)\end{bmatrix}.
\end{align}
We see that the $(2,1)$ block of $\eu^{\tilde{\bF}(\btheta) (t-t')}$ is the desired partial derivative of $\eu^{\bF(\btheta) (t-t')}$ with respect to $\theta_k$, i.e.,
\[ \frac{\partial }{\partial \theta_k}\eu^{\bF(\btheta) (t-t')} = \begin{bmatrix} \bzero & \bI\end{bmatrix} \eu^{\tilde{\bF}(\btheta)(t-t')} \begin{bmatrix} \bI \\ \bzero \end{bmatrix},\]
which leads to the symmetric SSS representation of $\partial \bK / \partial \theta_k$,
\begin{align}\label{eq:case2:generators:symmetric}
	\journallayout{\begin{aligned}}{}
        \bU_i & = \begin{bmatrix}\bzero & \bC(t_i)\end{bmatrix}, \quad
        \bT_i = \eu^{\btF(\btheta) \Delta t_i},  \quad
        \bV_i = \bC(t_i)\begin{bmatrix} \bP(t_i;\btheta) & \frac{\partial\bP(t_i;\btheta)}{\partial \theta_k}\end{bmatrix}\bT_i^\top, \journallayout{\\}{&}
        \bD_i & = \bC(t_i)\frac{\partial\bP(t_i;\btheta)}{\partial \theta_k}\bC(t_i)^\top, \journallayout{\quad i\in \{1,\ldots,n\}.}{}
    \journallayout{\end{aligned}}{}
\end{align}
\journallayout{}{for $i\in \{1,\ldots,n\}$.} We note that the SSS order of this presentation of $\partial \bK / \partial \theta_k$ is twice that of $\bK$.

For the Matérn and celerite kernels, where we have closed-form expressions for $\bP_\infty(\btheta)$, we can compute $\partial \bP_{\infty}(\btheta)/\partial \theta_k$ analytically. Similarly, for the spline kernel, the state covariance $\bP(t;\btheta)$ only depends on $\sigma^2$, and $\partial \bP(t;\btheta)/\partial \sigma^2$ simplifies to $\bP(t;\btheta)/\sigma^2$. More generally, for a stationary kernel, $\partial \bP_\infty(\btheta)/\partial \theta_k$ can be computed by differentiating the Lyapunov equation $\bF(\btheta)\bP_\infty(\btheta) + \bP_\infty(\btheta)\bF(\btheta)^\top + \bG(\btheta) \bQ_{\bw}\bG(\btheta)^\top = \bzero$ with respect to $\theta_k$, which yields
\journalmultline{
	\bzero = \bF(\btheta)\frac{\partial \bP_\infty(\btheta)}{\partial \theta_k} + \frac{\partial \bP_\infty(\btheta)}{\partial \theta_k}\bF(\btheta)^\top
}{
	{}+ \frac{\partial \bF(\btheta)}{\partial \theta_k}\bP_\infty(\btheta) + \bP_\infty(\btheta)\frac{\partial \bF^\top(\btheta)}{\partial \theta_k} + \frac{\partial (\bG(\btheta)\bQ_{\bw}\bG(\btheta)^\top)}{\partial \theta_k}.
}
This is a Lyapunov equation in $\partial \bP_\infty/\partial \theta_k$ that can be solved numerically using the Bartels--Stewart algorithm \cite{BaS:72}. In the nonstationary case, $\partial \bP(t;\btheta)/\partial \theta_k$ can be computed by differentiating the differential Lyapunov equation $\frac{\dr}{\dr t}\bP(t;\btheta) = \bF(\btheta)\bP(t;\btheta) + \bP(t;\btheta)\bF(\btheta)^\top + \bG(\btheta)\bQ_{\bw}\bG(\btheta)^\top$ with respect to $\theta_k$, which yields
\journalmultline{
	\frac{\dr}{\dr t}\frac{\partial \bP(t;\btheta)}{\partial \theta_k} = \bF(\btheta)\frac{\partial \bP(t;\btheta)}{\partial \theta_k} + \frac{\partial \bP(t;\btheta)}{\partial \theta_k}\bF(\btheta)^\top
}{
	{}+ \frac{\partial \bF(\btheta)}{\partial \theta_k}\bP(t;\btheta) + \bP(t;\btheta)\frac{\partial \bF^\top(\btheta)}{\partial \theta_k} + \frac{\partial (\bG(\btheta)\bQ_{\bw}\bG(\btheta)^\top)}{\partial \theta_k}.
}
This is a linear ordinary differential equation (ODE) in $\partial \bP(t;\btheta)/\partial \theta_k$ with initial condition $\partial \bP(t_0;\btheta)/\partial \theta_k$, and it can be solved numerically using standard ODE solvers.

\subsubsection*{Case 3: time warping}
Our last case considers the situation where the hyperparameter $\theta_k$ enters the model through a monotone time warp function $\phi(t;\theta_k)$. Specifically, we will assume that the kernel is given by $\check{k}(t,t') = k(\phi(t;\theta_k),\phi(t';\theta_k))$, where $k$ is an LTI-SDE kernel that does not depend on $\theta_k$. Assuming that $\bC(\tau)$ is continuously differentiable and letting $\tau = \phi(t;\theta_k)$ and $\tau' = \phi(t';\theta_k)$, the chain rule yields
\begin{align}\label{eq:kernel:deriv:timewarp:chain}
    \frac{\partial}{\partial \theta_k} \check{k}(t,t';\btheta) & = a(t)\frac{\partial k(\tau,\tau')}{\partial \tau} + a(t')\frac{\partial k(\tau,\tau')}{\partial \tau'},
\end{align}
where $a(t) = \partial\phi(t;\theta_k)/\partial \theta_k = \partial \tau/\partial \theta_k$. To derive an SSS representation of $\partial \bK/\partial \theta_k$, it suffices to consider $\tau \geq \tau'$ and apply the chain rule to the definition of $k$ in \eqref{eq:sde:nonstationary:cov}, which leads to
\begin{align}
    \frac{\partial k(\tau,\tau')}{\partial \tau} & = \bH(\tau)\eu^{\bF(\tau-\tau')}\bP(\tau')\bC(\tau')^\top,
\end{align}
where $\bH(\tau) = \bC(\tau)\bF + \frac{\partial\bC(\tau)}{\partial\tau}$. Similarly, using the differential Lyapunov equation $\partial \bP(\tau')/\partial \tau' = \bF\bP(\tau') + \bP(\tau')\bF^\top + \bG\bQ_\bw\bG^\top$, we have
\begin{align}
    \frac{\partial k(\tau,\tau')}{\partial \tau'} & =\bC(\tau)\eu^{\bF(\tau-\tau')}\bP(\tau') \bH(\tau')^\top + \bC(\tau)\eu^{\bF(\tau-\tau')} \bG\bQ_{\bw}\bG^\top \bC(\tau')^\top.
\end{align}
Combining these two expressions with \eqref{eq:kernel:deriv:timewarp:chain} yields
\begin{align*}
	    & \frac{\partial}{\partial \theta_k} \check{k}(t,t';\btheta) = \journallayout{\\ & \quad}{}
	    \begin{bmatrix}
	    a(t) \bH(\tau) &  \bC(\tau) \\
	    \end{bmatrix} (\bI_2\otimes \eu^{\bF(\tau-\tau')})
	    \begin{bmatrix}
	    \bP(\tau') \bC(\tau')^\top \\
	    a(t') \left(\bP(\tau') \bH(\tau')^\top + \bG\bQ_{\bw}\bG^\top \bC(\tau')^\top\right)
    \end{bmatrix},
\end{align*}
from which we find an SSS representation of $\partial \bK/\partial \theta_k$ with generators
\begin{align*}
	\bU_i &= \begin{bmatrix}
    a(t_i) \bH(\tau_i) &  \bC(\tau_i) \\
    \end{bmatrix}, &
	\bT_i &= \bI_2 \otimes \eu^{\bF \Delta \tau_i}, &
	\bV_i &= \bU_i\bB_i\bT_i^\top, &
	\bD_i &= \bU_i\bB_i\bU_i^\top,
\end{align*}
for $i \in \{1,\ldots,n\}$, where
\[\bB_i = \begin{bmatrix}
         \bzero & \bP(\tau_i) \\ \bP(\tau_i) & a(t_i)\bG\bQ_\bw\bG^\top
     \end{bmatrix}.\]
Like in case 2, the SSS order of this presentation of $\partial \bK / \partial \theta_k$ is twice that of $\bK$.

The expression for $\partial \bK/\partial \theta_k$ can be simplified further if $\bC(\tau)$ is constant or if the process is stationary, in which case $\bP(\tau)$ can be replaced with $\bP_\infty$.

\subsection{Posterior mean and covariance}\label{sub:sec:posterior}

We now present methods for evaluating the posterior mean, matrix-vector products with the posterior covariance and precision, and the diagonal blocks of the posterior covariance. Recall that both the training inputs $\bt = (t_1,\ldots,t_n)$ and the test inputs $\bt_* = (t_{*1},\ldots,t_{*\ntest})$ are assumed to be nondecreasing, which implies that $\bK$ and $\bK_{**}$ are SSS matrices. However, the augmented input vector $\tilde{\bt} = (\bt, \bt_*)$ is not necessarily nondecreasing, so $\bK_*$ is not SSS in general. To address this, we introduce a permutation matrix $\bm{\Pi}$ of order $N = n+\ntest$ such that
\[ \hat \bt = \bm{\Pi} \begin{bmatrix} \bt \\ \bt_* \end{bmatrix} \]
is a vector of nondecreasing inputs, i.e., $\hat t_1 \leq \hat t_2 \leq \dots \leq \hat t_{N}$. We also define an augmented kernel matrix $\hat{\bK}$ of order $Nd$ with blocks $\hat{\bK}_{ij} = k(\hat t_i, \hat t_j)$ and an augmented noise covariance matrix $\hat{\bm{\Sigma}}(\gamma)$ of order $Nd$ with diagonal blocks
\[
	\hat{\bm{\Sigma}}_{i}(\gamma) =
	\begin{cases}
		\bSigma_i,    & i \in \cI,   \\
		\gamma \bI_d, & i \in \cI_*,
	\end{cases}
\]
where the index sets $\cI$ and $\cI_*$ form a disjoint partition of $\{1,\ldots,N\}$, corresponding to training and test inputs, respectively. Thus, we have
\begin{align}\label{eq:Khat}
	\hat{\bK} = (\bm{\Pi} \otimes \bI_d) \begin{bmatrix} \bK & \bK_*^\top \\ \bK_* & \bK_{**} \end{bmatrix} (\bm{\Pi}^\top \otimes \bI_d), \qquad
	\hat{\bm{\Sigma}}(\gamma) = (\bm{\Pi} \otimes \bI_d) \begin{bmatrix} \bm{\Sigma} & \bzero \\ \bzero & \bm{\Sigma}_*(\gamma) \end{bmatrix} (\bm{\Pi}^\top \otimes \bI_d),\end{align}
where $\bm{\Sigma}_*(\gamma) = \blkdiag(\gamma \bI_d,\ldots,\gamma\bI_d)$.  We also define selection operators
\begin{align}\label{eq:selection_operator}
\cS = \begin{bmatrix} \bI_{nd} & \bzero_{\ntest d} \end{bmatrix} (\bm{\Pi}^\top \otimes \bI_d), \qquad \cS_* = \begin{bmatrix} \bzero_{nd} & \bI_{\ntest d} \end{bmatrix} (\bm{\Pi}^\top \otimes \bI_d), \end{align}
such that $\bK = \cS \hat{\bK} \cS^\top$, $\bK_* = \cS_* \hat{\bK} \cS^\top$, $\bK_{**} = \cS_* \hat{\bK} \cS_*^\top$, $\bm{\Sigma} = \cS \hat{\bm{\Sigma}}(\gamma) \cS^\top$, and $\bm{\Sigma}_*(\gamma ) = \cS_* \hat{\bm{\Sigma}}(\gamma) \cS_*^\top$.
We note that the permutation guarantees that $\hat{\bK}$ and $\hat{\bK} + \hat{\bm{\Sigma}}(\gamma)$ are symmetric SSS matrices.

\subsubsection*{Posterior mean}

The posterior mean can then be expressed as
\begin{align}\label{eq:gp:posterior:mean:sss}
	\bm{\mu}_{\bff_*|\by} = \cS_* \hat{\bK} \cS^\top \bm{\alpha}, \qquad \bm{\alpha} = (\bK + \bm{\Sigma})^{-1}\by,
\end{align}
which can be evaluated in $\cO(N(d^3+p^3))$ time by first solving the linear system $(\bK + \bm{\Sigma})\bm{\alpha} = \by$ using the SSS solver and then computing the matrix-vector product $\cS_* \hat{\bK} \cS^\top \bm{\alpha}$ using the SSS matrix-vector algorithm.

\subsubsection*{Matrix-vector products for posterior covariance and precision}
Matrix-\allowbreak{}vector products with the posterior covariance $\bm{\Sigma}_{\bff_*|\by}$ can be evaluated in $\cO(N(d^3+p^3))$ time using the expression
\begin{align}\label{eq:gp:posterior:cov:matvec}
	\bm{\Sigma}_{\bff_*|\by} \bv & = \bK_{**} \bv - \cS_* \hat{\bK} \cS^\top (\bK+\bm{\Sigma})^{-1} \cS \hat{\bK} \cS_*^\top \bv,
\end{align}
which uses the same SSS solver and matrix-vector product algorithms as in \eqref{eq:gp:posterior:mean:sss}.
Similarly, if $\bm{\Sigma}_{\bff_*|\by}$ is positive definite, then we can also compute matrix-vector products with the posterior precision $\bm{\Sigma}_{\bff_*|\by}^{-1}$ in $\cO(N(d^3+p^3))$ time by noting that
\begin{align}\label{eq:gp:posterior:prec:matvec}
	\bm{\Sigma}_{\bff_*|\by}^{-1} = \cS_* (\hat{\bK} + \hat{\bm{\Sigma}}(0))^{-1} \cS_*^\top.
\end{align}
This expression follows from the block-matrix inversion formula, which can be derived from the block factorization
\begin{align*}
	\begin{bmatrix}
		\bK +\bm{\Sigma} & \bK_*^\top \\
		\bK_*            & \bK_{**}
	\end{bmatrix} =
	\begin{bmatrix}
		\bI                           & \bzero \\
		\bK_*(\bK + \bm{\Sigma})^{-1} & \bI
	\end{bmatrix}
	\begin{bmatrix}
		\bK + \bm{\Sigma} & \bzero                                             \\
		\bzero            & \bm{\Sigma}_{\bff_*|\by}
	\end{bmatrix}
		\begin{bmatrix}
			\bI    & (\bK + \bm{\Sigma})^{-1}\bK_*^\top \\
			\bzero & \bI
	\end{bmatrix},
\end{align*}
where $\bm{\Sigma}_{\bff_*|\by} = \bK_{**} - \bK_*(\bK + \bm{\Sigma})^{-1}\bK_*^\top$. If $\bm{\Sigma}_{\bff_*|\by}$ is positive definite, then the above factorization is invertible and we have
\begin{align*}
	\bm{\Sigma}_{\bff_*|\by}^{-1} =
	\begin{bmatrix} \bzero & \bI \end{bmatrix}
	\begin{bmatrix}
		\bK +\bm{\Sigma} & \bK_*^\top \\
		\bK_*            & \bK_{**}
	\end{bmatrix}^{-1}
	\begin{bmatrix} \bzero \\ \bI \end{bmatrix}
	= \cS_* (\hat{\bK} + \hat{\bm{\Sigma}}(0))^{-1} \cS_*^\top.
\end{align*}

\subsubsection*{Diagonal blocks of posterior covariance}

The diagonal blocks of $\bm{\Sigma}_{\bff_*|\by}$ are often of particular interest but require a different approach, which we now describe. First, note that
\begin{align*} \lim_{\gamma \to \infty}
	\begin{bmatrix}
		\bK +\bm{\Sigma} & \bK_*^\top                         \\
		\bK_*            & \bK_{**} + \bm{\Sigma}_{*}(\gamma)
	\end{bmatrix}^{-1} =
	\begin{bmatrix}
		(\bK + \bm{\Sigma})^{-1}       & \bzero \\
		\bzero & \bzero
	\end{bmatrix},
\end{align*}
which leads to the following expression for the posterior covariance:
\[\bm{\Sigma}_{\bff_*|\by} = \bK_{**} - \bK_*(\bK + \bm{\Sigma})^{-1}\bK_*^\top = \lim_{\gamma\to\infty}  \cS_*(\hat \bK -\hat \bK (\hat \bK + \hat{\bm{\Sigma}}(\gamma))^{-1} \hat \bK )\cS_*^\top. \]
Letting $\hat{\bM}(\gamma) = \hat \bK + \hat{\bm{\Sigma}}(\gamma)$ such that $\hat \bK = \hat{\bM}(\gamma) - \hat{\bm{\Sigma}}(\gamma)$, we obtain
\begin{align}\label{eq:gp:posterior:cov:perturbed}
	\hat \bK -\hat \bK (\hat \bK + \hat{\bm{\Sigma}}(\gamma))^{-1} \hat \bK = \hat{\bm{\Sigma}}(\gamma) -\hat{\bm{\Sigma}}(\gamma)\hat{\bM}(\gamma)^{-1}\hat{\bm{\Sigma}}(\gamma).
\end{align}
The basic idea is to compute the diagonal blocks of the right-hand side of \eqref{eq:gp:posterior:cov:perturbed} recursively using the SSS structure of $\hat{\bM}(\gamma)$ and then take the limit as $\gamma \to \infty$. \Cref{app:posterior:variance} provides the details of this derivation, which  combines the SSS Cholesky factorization algorithm in \Cref{alg:cholesky} with the diagonal inverse algorithm in \Cref{alg:diagonal:inverse}, leading to \Cref{alg:posterior:variance}. The computational complexity of the algorithm is $\cO(N(d^3+p^3))$.

\begin{algorithm}[H]
	\caption{Diagonal blocks of posterior covariance}
	\label{alg:posterior:variance}
	\KwIn{SSS matrix $\hat{\bK}$ of the form \eqref{eq:Khat} with SSS generators $(\bU, \bT, \bV, \bD)$, $\hat{\bm{\Sigma}}(0)$, and  index set $\cI$ containing the indices of the training inputs.}
	\KwOut{Diagonal blocks $\hat{\bm{\Omega}}_1,\ldots,\hat{\bm{\Omega}}_N$ of $\lim_{\gamma\to\infty} (\hat{\bK} - \hat{\bK}(\hat{\bK} + \hat{\bm{\Sigma}}(\gamma))^{-1}\hat{\bK})$.}
	$\bm{\Psi}_1 \gets \bzero_{p\times p}$\;
	\For{$i=1$ \KwTo $N$}{
		\eIf{$i \in \cI$}{
			$\btD_i \gets \chol(\bD_i - \bU_i\bm{\Psi}\bU_i^\top + \hat{\bm{\Sigma}}_i)$ \tcp*{Lower triangular Cholesky factor}
			$\btV_i^\top \gets (\bV_i^\top - \bT_i\bm{\Psi}\bU_i^\top)\btD_i^{-\top}$\;
			$\boU_i \gets -\btD_i^{-1}\bU_i$, $\boT_i \gets \bT_i + \btV_i^\top \boU_i$, $\bA_i \gets \btD_i^{-\top}\btD_i^{-1}\hat{\bm{\Sigma}}_i$\;
		}{
			$\btV_i^\top \gets \bzero$, $\boU_i \gets \bzero$, $\boT_i \gets \bT_i$, $\bA_i \gets \bI$\;
		}
		$\bm{\Psi}_{i+1} \gets \bT_i\bm{\Psi}_i\bT_i^\top + \btV_i^\top\btV_i$\;
	}

	$\bm{\Phi} \gets \bzero_{p\times p}$\;
	\For{$i = N$ \KwTo $1$}{
		$\bW_i^\top \gets (\bV_i^\top - \bT_i\bm{\Psi}_i\bU_i^\top)\bA_i$\;
		$\hat{\bm{\Omega}}_i \gets (\bD_i - \bU_i\bm{\Psi}_i\bU_i^\top)\bA_i - \bW_i\bm{\Phi}\bW_i^\top$\;
		$\bm{\Phi} \gets \boT_i^\top\bm{\Phi}\boT_i + \boU_i^\top\boU_i$\;
	}
	\Return{$(\hat{\bm{\Omega}}_1,\ldots,\hat{\bm{\Omega}}_N)$}
\end{algorithm}

\subsection{Sampling from the posterior distribution}\label{sub:sec:posterior:sampling}
Samples from the posterior distribution \eqref{eq:gp:posterior} can be obtained by first sampling from the joint prior distribution of $\by$ and $\bff_*$ in \eqref{eq:joint:distribution} and then conditioning on the observed data, which is a technique known as Matheron’s update rule \cite{wilson2021}. First, we draw a sample from the joint prior distribution of $\by$ and $\bff_*$, i.e.,
\begin{equation}
	\mathcal{N}\left(\bzero, \begin{bmatrix}\bK + \bm{\Sigma} & \bK_*^\top \\ \bK_* & \bK_{**}\end{bmatrix}\right).
\end{equation}
Assuming that this is a nondegenerate distribution, we readily obtain a sample $(\tilde\by, \tilde \bff_*)$ as
\begin{align*}
	\begin{bmatrix}\tilde\by \\ \tilde \bff_* \end{bmatrix} =
	(\bm{\Pi}^\top \otimes \bI_d) \hat{\bL} \bm{\xi}, \quad \bm{\xi} \sim \mathcal{N}(\bzero, \bI),
\end{align*}
where $\hat{\bL}$ denotes the Cholesky factor of the SSS matrix $\hat{\bK} + \hat{\bm{\Sigma}}(0)$. A sample $\tilde \bz$ from the desired posterior distribution may then be obtained by conditioning on the observations $\by$, i.e.,
\begin{align}\label{eq:gp:posterior:sample}
	\tilde \bz  = \tilde{\bff}_* + \bK_*(\bK + \bm{\Sigma})^{-1}(\by - \tilde\by)  =\tilde{\bff}_* + \cS_* \hat{\bK}\cS^\top (\bK + \bm{\Sigma})^{-1}(\by - \tilde\by).
\end{align}
Given SSS Cholesky factors of $\bK+\bm{\Sigma}$ and $\hat{\bK} + \hat{\bm{\Sigma}}(0)$, which can be computed once and reused for multiple samples, a sample $\tilde \bz$ can be computed in $\cO(N(d^2+p^2))$ time using the SSS solver and matrix-vector product algorithms.

\section{Numerical experiments}\label{sec:numerical_experiments}

To validate the proposed SSS representation of LTI-SDE-induced kernels, we perform a series of numerical experiments to assess the numerical accuracy of the proposed algorithms. The experiments are based on a Julia implementation of the SSS algorithms, which is available in the \texttt{SSSMatrices.jl} package.\footnote{Available on GitHub: \url{https://github.com/mipals/SSSMatrices.jl}} In the first two experiments, we consider the kernels described in \Cref{sec:sde-to-sss}, evaluated at $n=100$ training inputs drawn from the standard uniform distribution and subsequently sorted in ascending order. The kernels are scalar-valued ($d=1$), and hence $\bK$ is $100\times 100$. \Cref{tab:configuration} provides an overview of the kernels and the corresponding hyperparameter values that we use for the experiments.
\begin{table}[htbp]
	\centering
	\begin{tabular}{llrr}
		\toprule
		Kernel                & Hyperparameters                                 & \(\kappa(\bK)\)     & \(\kappa(\bK+\sigma_{\mathrm{n}}^2\bI)\) \\
		\midrule
		Matérn ($m=2$)        & \(\btheta = (\sigma,\alpha)=(1,2)\)             & $3.2 \cdot 10^{18}$ & $9.2 \cdot 10^{5}$                       \\
		Celerite              & \(\btheta = (\sigma,\omega_0,Q)=(1,1.2,0.6)\)   & $1.4 \cdot 10^{13}$ & $1.6 \cdot 10^{5}$                       \\
		Diagonal-correlated   & \(\btheta = (\sigma,\alpha,\beta)=(1,0.5,0.6)\) & $5.3 \cdot 10^{6}$  & $4.7 \cdot 10^{5}$                       \\
		Tuned-correlated      & \(\btheta = (\sigma,\alpha)=(1,0.6)\)           & $6.7 \cdot 10^{6}$  & $6.3 \cdot 10^{5}$                       \\
		Spline ($m=2$)        & \(\btheta = (\sigma)=(1)\)                      & $6.6 \cdot 10^{12}$ & $7.4 \cdot 10^{4}$                       \\
		Stable spline ($m=2$) & \(\btheta = (\sigma,\alpha)=(1,0.3)\)           & $1.7 \cdot 10^{15}$ & $2.2 \cdot 10^{5}$                       \\
		\bottomrule
	\end{tabular}
	\medskip

	\caption{Hyperparameters used for each kernel in the numerical experiments. We use $\sigma_{\mathrm{n}}=0.01$ in all experiments. The condition numbers in columns 3--4 were computed using Float64.}
	\label{tab:configuration}
\end{table}
The third experiment illustrates the use of the SSS representation for computing posterior means and covariances in a GP regression setting with a matrix-valued kernel ($d=2$).

\subsection{Experiment 1: Accuracy}
The accuracy of the SSS representation is assessed by comparing its Float64 matrix-vector products, linear-system solutions, and log determinants with dense reference computations based on Julia’s BigFloat type at 256-bit precision, corresponding to a machine epsilon of $2^{-255}\approx 10^{-77}$.

To establish a baseline for the relative error expected from Float64 arithmetic, we also compare the same operations performed using dense Float64 matrices with the high-precision reference. The results are summarized in \Cref{tab:correctness}, which shows that the SSS representation achieves relative errors on the same order of magnitude as the dense Float64 computations. The relative errors displayed in \Cref{tab:correctness} are computed as
$\|g(\bK) - g(\bK_{\text{big}})\|_2/\|g(\bK_{\text{big}})\|_2$,
where $g(\cdot)$ denotes the matrix operation of interest, and $g(\bK)$ and $g(\bK_{\text{big}})$ are the results obtained using Float64 and BigFloat, respectively.
For the matrix-vector product, $g(\bK) = \bK \bx$, and for the linear solve, $g(\bK)=\bK^{-1}\bx$, we generate a random $\bx$ with independent standard normal entries. The results in \Cref{tab:correctness} show that SSS computations and the dense Float64 computations have similar relative errors compared to the high-precision reference.
\begin{table}[htbp]
	\centering
	\journallayout{\begingroup\renewcommand{\footnotesize}{\fontsize{6}{7}\selectfont}}{}
	\begin{tabular}{lrrrrrr}
		\toprule
		Kernel
		                      & \multicolumn{2}{c}{Matrix-vector product}
		                      & \multicolumn{2}{c}{Solve}
		                      & \multicolumn{2}{c}{Log determinant}                                                                                                                             \\
		\cmidrule(lr){2-3}
		\cmidrule(lr){4-5}
		\cmidrule(lr){6-7}
		                      & \multicolumn{1}{c}{Dense}                 & \multicolumn{1}{c}{SSS}
		                      & \multicolumn{1}{c}{Dense}                 & \multicolumn{1}{c}{SSS}
		                      & \multicolumn{1}{c}{Dense}                 & \multicolumn{1}{c}{SSS}                                                                                             \\
		\midrule
		Matérn                & $5.2 \cdot 10^{-15}$ & $5.3 \cdot 10^{-15}$ & $1.7 \cdot 10^{-11}$ & $6 \cdot 10^{-12}$   & $7.7 \cdot 10^{-15}$ & $8.9 \cdot 10^{-14}$ \\
        Celerite              & $5.4 \cdot 10^{-15}$ & $1.6 \cdot 10^{-14}$ & $4.3 \cdot 10^{-12}$ & $4.9 \cdot 10^{-13}$ & $6.5 \cdot 10^{-15}$ & $9.5 \cdot 10^{-15}$ \\
        Diagonal-correlated   & $1.8 \cdot 10^{-15}$ & $2.7 \cdot 10^{-15}$ & $1.7 \cdot 10^{-12}$ & $2.2 \cdot 10^{-13}$ & $2.1 \cdot 10^{-15}$ & $9.8 \cdot 10^{-16}$ \\
        Tuned-correlated      & $7.6 \cdot 10^{-16}$ & $7.1 \cdot 10^{-16}$ & $4.3 \cdot 10^{-13}$ & $2.7 \cdot 10^{-13}$ & $1.8 \cdot 10^{-15}$ & $2.5 \cdot 10^{-15}$ \\
	    Spline                  & $1.5 \cdot 10^{-15}$ & $7.5 \cdot 10^{-16}$ & $1.1 \cdot 10^{-12}$ & $2.7 \cdot 10^{-13}$ & $2.2 \cdot 10^{-15}$ & $2.3 \cdot 10^{-15}$ \\
	    Stable spline           & $8.1 \cdot 10^{-15}$ & $2.1 \cdot 10^{-15}$ & $5.5 \cdot 10^{-12}$ & $6.8 \cdot 10^{-13}$ & $4 \cdot 10^{-15}$   & $1.5 \cdot 10^{-14}$ \\
		\bottomrule
	\end{tabular}
	\journallayout{\endgroup}{}
	\medskip
	\caption{Relative errors of matrix operations computed using dense and SSS representations, relative to results obtained with BigFloat at 256-bit precision. The hyperparameters are listed in \Cref{tab:configuration}. The matrix-vector product and solve operations are $\bK \bx$ and $\bK^{-1}\bx$, respectively, whereas the log determinant is $\log\det(\bK + \sigma_{\mathrm{n}}^2\,\bI)$.}
	\label{tab:correctness}
\end{table}

\subsection{Experiment 2: Gradient validation}
To assess the accuracy of the SSS representations of the kernel derivative matrices, we compute the partial derivatives of the log marginal likelihood with respect to the hyperparameters. As in the previous experiment, we compare Float64 SSS computations with dense reference computations performed in both Float64 and BigFloat arithmetic. \Cref{tab:gradient} presents the absolute and relative errors with respect to the BigFloat reference. We note that in each case, an SSS representation of the partial derivative of the kernel matrix is computed using one of the techniques presented in \Cref{sub:sec:analytical:derivatives}.
\begin{table}[htbp]
	\centering
	\begin{tabular}{lcrrrr}
		\toprule
		Kernel                & $\theta_k$
		                      & \multicolumn{2}{c}{Absolute error}
		                      & \multicolumn{2}{c}{Relative error}                                                                                                     \\
		\cmidrule(lr){3-4}
		\cmidrule(lr){5-6}
		                      &                                    & \multicolumn{1}{c}{Dense} & \multicolumn{1}{c}{SSS}
		                      & \multicolumn{1}{c}{Dense}          & \multicolumn{1}{c}{SSS}                                                                           \\
		\midrule
		Matérn ($m=2$)        & \(\alpha\)                         & $7.7 \cdot 10^{-7}$       & $4.2 \cdot 10^{-7}$     & $5.9 \cdot 10^{-11}$ & $5 \cdot 10^{-11}$   \\
		Celerite              & \(\omega_0\)                       & $3.7 \cdot 10^{-7}$       & $5.9 \cdot 10^{-8}$     & $6 \cdot 10^{-9}$    & $6.6 \cdot 10^{-10}$ \\
		                      & \(Q\)                              & $3.3 \cdot 10^{-7}$       & $7.4 \cdot 10^{-8}$     & $1.6 \cdot 10^{-9}$  & $1.7 \cdot 10^{-10}$ \\
		Diagonal-correlated   & \(\alpha\)                         & $1.1 \cdot 10^{-7}$       & $1.4 \cdot 10^{-8}$     & $1.3 \cdot 10^{-12}$ & $1.6 \cdot 10^{-13}$ \\
		                      & \(\beta\)                          & $4.9 \cdot 10^{-8}$       & $3.1 \cdot 10^{-8}$     & $1.7 \cdot 10^{-12}$ & $9 \cdot 10^{-13}$   \\
		Tuned-correlated      & \(\alpha\)                         & $4.2 \cdot 10^{-8}$       & $2.5 \cdot 10^{-8}$     & $9.8 \cdot 10^{-13}$ & $5.7 \cdot 10^{-13}$ \\
		Stable spline ($m=2$) & \(\alpha\)                         & $7.5 \cdot 10^{-7}$       & $1.5 \cdot 10^{-7}$     & $2.6 \cdot 10^{-11}$ & $3.1 \cdot 10^{-12}$ \\
		\bottomrule
	\end{tabular}
	\medskip

	\caption{Absolute and relative errors for partial derivatives of the log marginal likelihood, computed using dense and SSS representations, relative to results obtained with BigFloat at 256-bit precision. The hyperparameter values are listed in \Cref{tab:configuration}.}
	\label{tab:gradient}
\end{table}

\subsection{Experiment 3: Posterior distribution and sampling}
We now demonstrate the use of the SSS representation for Gaussian process regression based on the Matérn-$5/2$ kernel with hyperparameters $(\sigma,\alpha)=(1.0,2.0)$. We will simulate observations of the first two entries of the state vector, corresponding to the function value and its derivative, i.e., $\bC(t) = \bigl[\begin{smallmatrix} \be_1 & \be_2\end{smallmatrix}\bigr]^\top \in \reals^{2 \times 3}$. To generate a ground truth reference and training data, we first form a vector $\hat\bt$ of $N=135$ monotonically increasing and equidistant points in the interval $[1,10]$. The training points $\bt$ are the subset of $\hat\bt$ lying in $[1,3.5]\cup[6.5,10]$, and the test points are the remaining points of $\hat\bt$, lying in $(3.5,6.5)$. We then simulate noisy observations by computing the Cholesky factorization $\hat\bK=\hat\bL\hat\bL^\top$ and computing a sample $\hat{\bff}$ from $\cN(\bzero,\hat{\bK})$ as $\hat\bff = \hat \bL\bm{\xi}$, where $\bm{\xi}\sim\cN(\bzero, \bI)$.
Then, using the selection operators defined in \eqref{eq:selection_operator}, we define $\bff = \cS \hat{\bff}$, $\bff_* = \cS_*\hat{\bff}$, and $\by = \bff + \bm{\varepsilon}$, where the noise term $\bm{\varepsilon}$ is sampled from $ \cN(\bzero,\bI_n \otimes \diag(1.1,1.5))$.
Finally, we use \eqref{eq:gp:posterior:mean:sss} to compute the posterior mean and \Cref{alg:posterior:variance} to compute the posterior variance on the full grid $\hat\bt$, and we draw $10$ samples from the posterior distribution using the expression in \eqref{eq:gp:posterior:sample}. \Cref{fig:matern:posterior:inference:sampling} shows the 95\% credible interval, the samples, and the posterior mean for both the function value and the derivative.
\begin{figure}[htbp]
	\centering
	\includegraphics[width=1.0\linewidth]{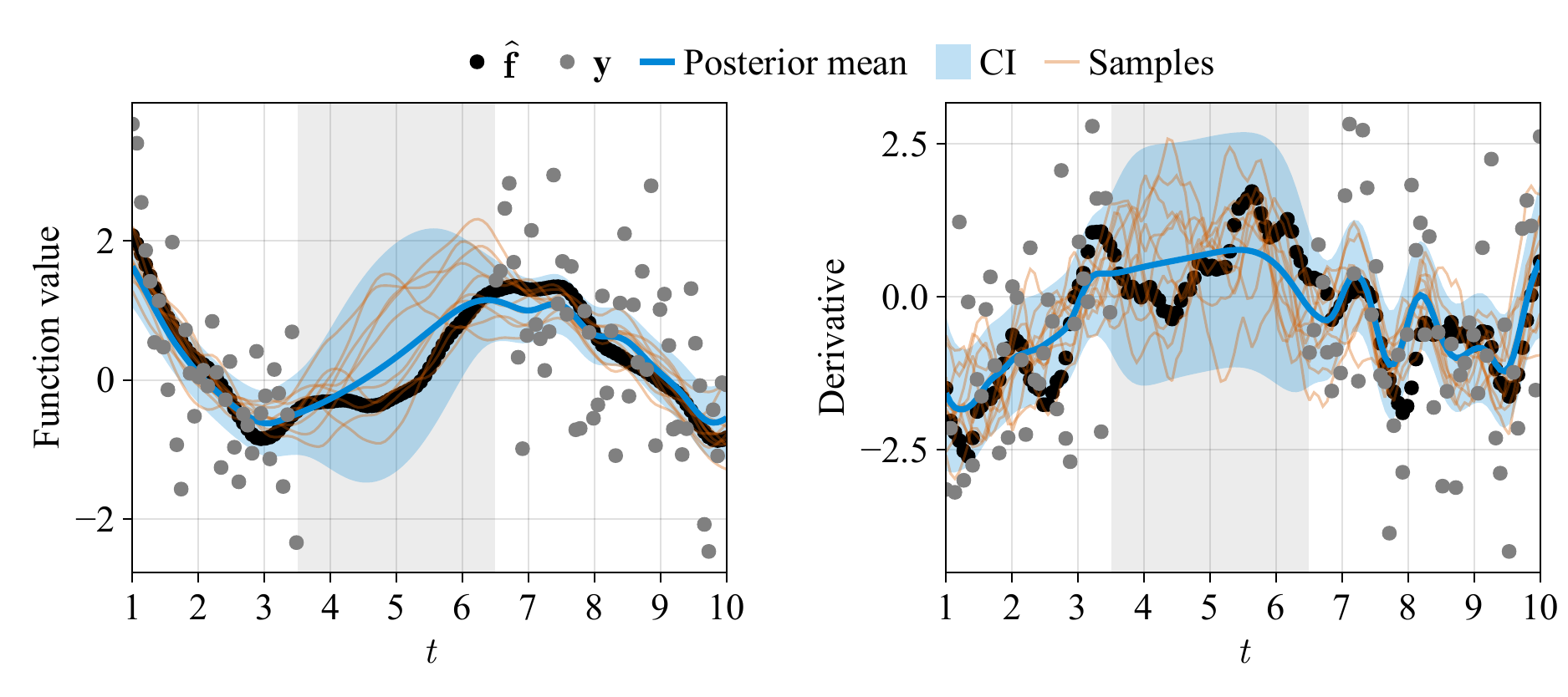}
	\caption{The 95\% credible interval, 10 samples, and posterior mean for a Matérn-$5/2$ GP with hyperparameters $(\sigma,\alpha)=(1.0,2.0)$. Both function value and derivative are observed.}
	\label{fig:matern:posterior:inference:sampling}
\end{figure}

\section{Conclusion}\label{sec:conclusion}
We have shown that kernels induced by LTI SDEs give rise to kernel matrices that naturally admit an SSS representation. Building on this representation, we have developed efficient recursive algorithms, with linear complexity in the number of data points, for all operations required for GP regression, including Cholesky and LDL$^\top$ factorization, forward and backward substitution, computation of the diagonal of the inverse, and the computation of the trace of the product of the inverse with an SSS matrix.

In contrast to previous applications of SSS representations to GP regression, we have also developed efficient methods for computing the analytic gradient of the log marginal likelihood. In particular, we derive a novel algorithm for $\tr(\bL^{-1}\bA\bL^{-\top})$ that operates directly on three SSS matrices, without forming intermediate SSS matrices. The resulting algorithms have been applied to several kernels, including the Matérn kernel with half-integer smoothness, the diagonal-correlated kernel, the spline kernel, and the stable spline kernel.

Importantly, our approach is not limited to any particular kernel family or estimation criterion. The algorithms apply to any kernel matrix that admits an SSS representation, including kernel matrices generated by both stationary and nonstationary processes, and the same recursions extend to analogous linear-algebra computations required in other kernel-based estimation problems. Moreover, the SSS structure naturally handles matrix-valued outputs through the block structure of the SSS generators, enabling joint observations of a function and its derivatives.

A natural direction for future work is to derive a corresponding trace algorithm for hierarchical semiseparable (HSS) matrices. This would extend the proposed framework to kernels that do not admit an SSS representation but can be efficiently approximated using the more general HSS structure.

\appendix
\crefalias{section}{appendix}
\crefalias{subsection}{appendix}
\crefalias{subsubsection}{appendix}
\section{Additional algorithms}\label{app:additional:algorithms}
In this appendix, we provide additional algorithms for forward and backward substitution and for computing an SSS representation of the inverse of a nonsingular, lower-triangular SSS matrix $\bL$. Moreover, we also state conditions under which $\bL^{-1}$ is lower block bidiagonal.

\subsection{Forward and backward substitution}\label{sub:sec:triangular:solve}
Given a nonsingular, lower triangular SSS matrix $\bL$ of the form \eqref{eq:L} with entries $\bL_{ij} = \bU_i\bT_{i-1:j+1}\btV_j^\top$ for $i > j$ and $\bL_{ii} = \btD_i$, the solution of the linear system of equations $\bL\bx = \vo$ can be computed efficiently using forward substitution. The algorithm is based on a standard state-space recursion with initial state $\vh_1 = \bzero$ and
\begin{align}\label{eq:forward:sub}
	\left.
	\begin{aligned}
		\bx_i     & = \btD_i^{-1}\left(\vo_i - \bU_i\vh_i\right) \\
		\vh_{i+1} & = \bT_i\vh_i + \btV_i^\top\bx_i
	\end{aligned}\right\} \quad i \in \{1, \ldots, n\}.
\end{align}
\Cref{alg:forward:sub} summarizes the forward substitution algorithm. Similarly, \Cref{alg:backward:sub} solves $\bL^\top \bx = \by$ via backsubstitution. Both algorithms have a computational complexity of $\cO(n(d^2+p^2))$.

\begin{algorithm}[H]
	\caption{Solve $\bL\bx=\vo$ via forward substitution}
	\label{alg:forward:sub}
	\KwIn{$\vo\in\reals^{nd}$ and SSS representation of $\bL$ of the form \eqref{eq:L}.}
	\KwOut{$\bx = \bL^{-1}\vo$.}
	$\vh = \bzero_{p}$\;
	\For{$i = 1$ \KwTo $n$}{
		$\bx_i \gets \btD_i^{-1} (\vo_i - \bU_i\vh)$\;
		$\vh \gets \bT_i\vh + \btV_i^\top\bx_i$\;
	}
\end{algorithm}

\begin{algorithm}[H]
	\caption{Solve $\bL^\top\bx=\vo$ via backward substitution}
	\label{alg:backward:sub}
	\KwIn{$\vo\in\reals^{nd}$ and SSS representation of $\bL$ of the form \eqref{eq:L}.}
	\KwOut{$\bx = \bL^{-\top}\vo$.}
	$\vh = \bzero_{p}$\;
	\For{$i = n$ \KwTo $1$}{
		$\bx_i \gets \btD_i^{-\top} (\vo_i - \btV_i\vh)$\;
		$\vh \gets \bT_i^\top\vh + \bU_i^\top\bx_i$\;
	}
\end{algorithm}

\subsection{Inverse of lower-triangular SSS matrix}\label{sub:sec:explicit:inverse}

Let $\bL$ be a nonsingular, lower-triangular SSS matrix of the form \eqref{eq:L}. The inverse $\bL^{-1}$ is also lower-triangular and SSS representable. To see this, consider the state-space model defined by the forward substitution recursion \eqref{eq:forward:sub}. Substituting $\bx_i = \btD_i^{-1}(\vo_i - \bU_i\vh_i)$ into the state recursion $\vh_{i+1} = \bT_i\vh_i + \btV_i^\top\bx_i$ eliminates $\bx_i$ from the state equation, resulting in a state-space model with input $\vo_i$ and output $\bx_i$, i.e.,
\begin{align}\label{eq:inverse:ssm}
	\left.
	\begin{aligned}
		\bx_i     & = \underbrace{-\btD_i^{-1}\bU_i}_{\boU_i}\vh_i + \underbrace{\btD_i^{-1}}_{\boD_i}\vo_i                                     \\
		\vh_{i+1} & = \underbrace{(\bT_i - \btV_i^\top\btD_i^{-1}\bU_i)}_{\boT_i}\vh_i + \underbrace{\btV_i^\top\btD_i^{-1}}_{\boV_i^\top}\vo_i
	\end{aligned} \right\} \quad i \in \{1, \ldots, n\},
\end{align}
with initial state $\vh_1 = \bzero$. The output of this state-space model is $\bx = \bL^{-1}\vo$, and hence $\bL^{-1}$ has block entries
\[ \bL^{-1} =
	\begin{cases}
		\boU_i\boT_{i-1:j+1}\boV_j^\top, & i > j, \\
		\boD_i,                          & i = j, \\
		\bzero,                          & i < j.
	\end{cases}
\]
An equivalent form of the inverse can be found in \cite{chandrasekaran2003b}.

\begin{algorithm}[H]
	\caption{Generators of inverse of lower-triangular SSS matrix}
	\label{alg:explicit:inverse}
	\KwIn{Lower-triangular SSS matrix $\bL$ of the form \eqref{eq:L}, defined by block-diagonal matrices $(\bU,\bT,\btV,\btD)$.}
	\KwOut{Block-diagonal matrices $(\boU, \boT, \boV, \boD)$ such that $\bL^{-1} = \boU(\bI-\bN^p\boT)^{-1}\bN^p\boV^\top + \boD $.}
	\For{$i = 1$ \KwTo $n$}{
		$\boD_i \gets \btD_i^{-1}$\;
		$\boV_i^\top \gets \btV_i^\top\boD_i$\;
		$\boU_i \gets -\boD_i\bU_i$\;
		$\boT_i \gets \bT_i - \boV_i^\top\bU_i$\;
	}
\end{algorithm}

\subsection{Structure of inverse Cholesky factor}\label{app:block:tridiagonal}
Consider a positive definite kernel matrix $\bK$ of the form \eqref{eq:sss:symmetric} with generators $(\bU, \bT, \bV, \bD)$ defined by \eqref{eq:sss:generators:nonstationary}, i.e.,
\[ 	\bU_i = \bC_i, \quad
	\bT_i = \eu^{\bF \Delta t_{i}}, \quad
	\bV_i = \bC_i\bP_i \bT_i^\top, \quad
	\bD_i = \bC_i\bP_i\bC_i^\top, \quad i \in \{1,\ldots,n\}.\]
We will characterize when $\bK^{-1}$ is block tridiagonal or, equivalently, when the Cholesky factorization $\bK=\bL\bL^\top$ yields a Cholesky factor of the form \eqref{eq:L} whose inverse is block bidiagonal. From \eqref{eq:offdiag} and the definition of $\bV_i$, we have
\begin{align*}
	\btV_i^\top = (\bV_i^\top - \bT_i\bm{\Psi}_i\bC_i^\top)\btD_i^{-\top} = (\bT_i\bP_i\bC_i^\top - \bT_i\bm{\Psi}_i\bC_i^\top)\btD_i^{-\top} = \bT_i(\bP_i - \bm{\Psi}_i)\bC_i^\top\btD_i^{-\top},
\end{align*}
and from \eqref{eq:diag} and the definition of $\bU_i$, we have
\begin{align*}
	\btD_i\btD_i^\top = \bC_i\bP_i\bC_i^\top - \bC_i\bm{\Psi}_i\bC_i^\top = \bC_i(\bP_i - \bm{\Psi}_i)\bC_i^\top.
\end{align*}
Using these two expressions while defining
\[\bP = \blkdiag(\bP_1,\ldots,\bP_n),\qquad \bm{\Psi} = \blkdiag(\bm{\Psi}_1,\ldots,\bm{\Psi}_n),\]
we can rewrite $\bL$ as
\begin{align*}
	\bL
	 & = \bC(\bI-\bN^p\bT)^{-1}\bN^p\btV^\top + \btD                                                           \\
	 & = \bC(\bI-\bN^p\bT)^{-1}\bN^p\bT(\bP - \bm{\Psi})\bC^\top\btD^{-\top} + \btD                            \\
	 & = (\bC(\bI-\bN^p\bT)^{-1}\bN^p\bT(\bP - \bm{\Psi})\bC^\top + \btD\btD^\top)\btD^{-\top}                 \\
	 & = (\bC(\bI-\bN^p\bT)^{-1}\bN^p\bT(\bP - \bm{\Psi})\bC^\top + \bC(\bP - \bm{\Psi} )\bC^\top)\btD^{-\top} \\
	 & = \bC((\bI-\bN^p\bT)^{-1}\bN^p\bT + \bI)(\bP - \bm{\Psi})\bC^\top\btD^{-\top}.
\end{align*}
We then note that $(\bI-\bN^p\bT)^{-1}\bN^p\bT + \bI = (\bI-\bN^p\bT)^{-1}$, which can be verified using \eqref{eq:blk:bidiag}, leading to the simplified expression
\[ \bL = \bC(\bI-\bN^p\bT)^{-1}(\bP-\bm{\Psi})\bC^\top\btD^{-\top}. \]
All the factors in this expression, except for the factor $(\bI-\bN^p\bT)^{-1}$, are block diagonal matrices. Thus, if $\bC$ is nonsingular (i.e., the blocks $\bC_1,\ldots,\bC_n$ are square and full rank), then
\[ \bL^{-1} = \btD^\top \bC^{-\top} (\bP-\bm{\Psi})^{-1}(\bI -\bN^p\bT) \bC^{-1}\]
is block bidiagional and $\bK^{-1}$ is block tridiagonal.

\begin{remark}
	For LTI-SDE-induced kernels, the structure of $\bL$ highlights the correspondence between $\bK$ and the discrete-time model as
	\begin{align*}
		\bK = \bL\bL^\top = \bC(\bI - \bN^p\bT)^{-1}(\bP-\bm{\Psi})(\bI - \bN^p\bT)^{-\top}\bC^\top.
	\end{align*}
	In other words, we can view $\bK$ as the covariance matrix for the output $\by_k = \bC_k\bx_k$ of a discrete process $\bx_{k+1} = \bT_k\bx_k + \bm{\varepsilon}_k$ with $\bm{\varepsilon}_k \sim \cN(\bzero, \bP_k - \bm{\Psi}_k)$.
\end{remark}

\section{The LDL factorization}\label{app:sec:ldl}
The algorithms in \Cref{sec:algorithms,sec:gp:regression} are based on the Cholesky factorization $\bK = \bL\bL^\top$ of a symmetric positive definite SSS matrix $\bK$. Similar algorithms can be derived based on the LDL$^\top$ factorization $\bK = \bL\btD\bL^\top$, where $\bL$ is unit lower triangular and $\btD$ is block diagonal. This approach is similar to the innovation form of the Kalman filter, which is based on the LDL$^\top$ factorization of the covariance matrix of the innovations.

Given the SSS generators $(\bU,\bT,\bV,\bD)$ of a symmetric positive definite matrix $\bK$ of the form \eqref{eq:sss:symmetric}, the LDL$^\top$ factorization $\bK = \bL \btD\bL^{\top}$ yields a unit lower-triangular $\bL$ of the form
\begin{equation}\label{eq:L:ldl}
	\bL = \bU (\bI - \bN^p\bT)^{-1}\bN^p\btV^\top + \bI,
\end{equation}
with entries $\bL_{ij} = \bU_i\bT_{i-1:j+1}\btV_j^\top$ for $i > j$ and $\bL_{ii} = \bI_d$. Expanding $(\bL\btD\bL^\top)_{ij}$ for $i > j$ and using $\bT_{i-1:k+1}=\bT_{i-1:j+1}\bT_j\bT_{j-1:k+1}$ gives
\begin{align*}
	(\bL\btD\bL^\top)_{ij}
	 & = \sum_{k=1}^{j-1}\bU_i\bT_{i-1:k+1}\btV_k^\top\btD_k\btV_k\bT_{j-1:k+1}^\top\bU_j^\top + \bU_i\bT_{i-1:j+1}\btV_j^\top\btD_j \\
	 & = \bU_i\bT_{i-1:j+1}\left(\bT_j\bm{\Psi}_j\bU_j^\top + \btV_j^\top\btD_j\right),
\end{align*}
where $\bm{\Psi}_j = \sum_{k=1}^{j-1}\bT_{j-1:k+1}\btV_k^\top\btD_k\btV_k\bT_{j-1:k+1}^\top$. Matching with $\bK_{ij} = \bU_i\bT_{i-1:j+1}\bV_j^\top$ and the diagonal $(\bL\btD\bL^\top)_{ii} = \bU_i\bm{\Psi}_i\bU_i^\top + \btD_i = \bD_i$ yields the recursions
\begin{align*}
		\btD_i          & = \bD_i - \bU_i\bm{\Psi}_i\bU_i^\top,                                               \\
		\btV_i^\top     & = (\bV_i^\top - \bT_i\bm{\Psi}_i\bU_i^\top)\btD_i^{-1},                             \\
		\bm{\Psi}_{i+1} & = \bT_i\bm{\Psi}_i\bT_i^\top + \btV_i^\top\btD_i\btV_i, \quad \bm{\Psi}_1 = \bzero.
	\end{align*}
\Cref{alg:ldl} summarizes the LDL$^\top$ factorization, which has a computational complexity of $\cO(n(d^2+p^2))$.

\begin{algorithm}[H]
	\caption{LDL$^\top$ factorization of symmetric positive definite SSS matrix $\bK$}
	\label{alg:ldl}
	\KwIn{$\bK$ defined by block-diagonal matrices $(\bU,\bT,\bV,\bD)$.}
	\KwOut{Block-diagonal matrices $\btV$ and $\btD$ such that $\bL=\bU(\bI-\bN^p\bT)^{-1}\bN^p\btV^\top + \bI$ and $\bK = \bL\btD\bL^\top$.}
	$\bm{\Psi} \gets \bzero_{p \times p}$\;
	\For{$i = 1$ \KwTo $n$}{
		$\bM \gets \bT_i\bm{\Psi}$\;
		$\btV_i^\top \gets \bV_i^\top - \bM\bU_i^\top$\;
		$\btD_i \gets \bD_i - \bU_i\bm{\Psi}\bU_i^\top$\;
		$\btV_i^\top \gets \btV_i^\top\btD_i^{-1}$\;
		$\bm{\Psi} \gets \bM\bT_i^\top + \btV_i^\top\btD_i\btV_i$\;
	}
    \Return{$(\btV_1,\ldots,\btV_n)$, $(\btD_1,\ldots,\btD_n)$}
\end{algorithm}

\subsection{SSS representation of inverse of $\bL$}
An SSS representation of $\bL^{-1}$ can be computed using \Cref{sub:sec:explicit:inverse} by letting $\btD_i = \bI$, i.e.,
\begin{align}\label{eq:ldl:Linv}
\bL^{-1} = \boU (\bI -\bN^p \boT)^{-1}\bN^p\boV^\top + \boD,
\end{align}
where $\boV_i = \btV_i$, $\boU_i = -\bU_i$, $\boT_i = \bT_i - \btV_i^\top\bU_i$, and $\boD_i=\bI$ for $i\in\{1,\ldots,n\}$.

\subsection{Diagonal blocks of $\bK^{-1}$}\label{sub:sec:diagonal:ldl}
The general SSS matrix-matrix multiplication algorithm of \cite{chandrasekaran2003b} cannot readily be applied for the computation of the diagonal blocks of $\bK^{-1}$ based on the LDL$\top$ factorization of $\bK$. However, the idea behind \Cref{alg:diagonal:inverse} can be used to derive a similar algorithm based on the LDL$^\top$ factorization without forming any intermediate SSS matrices. Given the LDL$^\top$ factorization $\bK = \bL\btD\bL^\top$ with $\bL$ unit lower triangular, the inverse is given by $\bK^{-1} = \bL^{-\top}\btD^{-1}\bL^{-1}$, where $\bL^{-1}$ is of the form \eqref{eq:ldl:Linv}. Using that $(\bL^{-1})_{ji} = \boU_j\boT_{j-1:i+1}\btV_i^\top$ for $j > i$ (which follows by noting that $\boV_i^\top = \btV_i^\top$), we have
\begin{equation}
	(\bK^{-1})_{ii} = \sum_{j=i}^{n}(\bL^{-1})_{ji}^\top\btD_j^{-1}(\bL^{-1})_{ji} =\btD_i^{-1} + \btV_i\bm{\Phi}_i\btV_i^\top,
\end{equation}
where $\bm{\Phi}_i = \sum_{j=i+1}^{n}\boT_{j-1:i+1}^\top\boU_j^\top\btD_j^{-1}\boU_j\boT_{j-1:i+1}$. Again, $\bm{\Phi}_i$ can be computed efficiently using a backward recursion
\begin{equation}
    \bm{\Phi}_n = \bzero, \qquad \bm{\Phi}_{i-1} = \boU_i^\top\btD_i^{-1}\boU_i + \boT_i^\top\bm{\Phi}_i\boT_i, \qquad i=n,\ldots,2.
\end{equation}
Compared to the Cholesky case, the term $\boD_i^\top\boD_i$ is replaced by $\btD_i^{-1}$.

\subsection{Trace of $\bK^{-1}\bA$}\label{sub:sec:trace:ldl}
\Cref{alg:trace} can be adapted to compute the trace of $\bK^{-1}\bA$ using the LDL$^\top$ factorization of $\bK$. Indeed, if $\bK = \bL \btD \bL^\top$, then
\begin{equation}
	\tr(\bK^{-1}\bA)
	= \tr\left(\btD^{-1}\bL^{-1}\bA\bL^{-\top}\right)
	= \sum_{i=1}^n \tr\left(\btD_i^{-1}\bm{\Theta}_i\right),
\end{equation}
where $\bL^{-1}$ is of the form \eqref{eq:ldl:Linv} and $\bm{\Theta}_i$ is the $i$th diagonal block of $\bL^{-1}\bA\bL^{-\top}$. This can be computed efficiently using a minor modification of \Cref{alg:trace}.

\subsection{Diagonal blocks of the posterior covariance}

\Cref{alg:posterior:variance} can be adapted to compute the diagonal blocks of the posterior covariance using the LDL$^\top$ factorization of $\hat{\bK}$. First, we take the limit of the LDL$^\top$ factorization as $\gamma\rightarrow\infty$, while storing the intermediate $\bm{\Psi}_i$. The diagonal blocks of the posterior covariance are then computed by using a backward recursion from $N$ to $1$. \Cref{alg:posterior:variance:ldl} summarizes the resulting algorithm.

\begin{algorithm}[H]
    \caption{Diagonal blocks of posterior covariance using the LDL$^\top$ factorization}
    \label{alg:posterior:variance:ldl}
    \KwIn{SSS matrix $\hat{\bK}$ of the form \eqref{eq:Khat} with SSS generators $(\bU, \bT, \bV, \bD)$, $\hat{\bm{\Sigma}}(0)$, and index set $\cI$ containing the indices of the training inputs.}
    \KwOut{Diagonal blocks $\hat{\bm{\Omega}}_1,\ldots,\hat{\bm{\Omega}}_N$ of $\lim_{\gamma\to\infty} (\hat{\bK} - \hat{\bK}(\hat{\bK} + \hat{\bm{\Sigma}}(\gamma))^{-1}\hat{\bK})$.}
    $\bm{\Psi}_1 \gets \bzero_{p\times p}$\;
    \For{$i=1$ \KwTo $N$}{
        \eIf{$i \in \cI$}{
            $\btD_i \gets \bD_i - \bU_i\bm{\Psi}_i\bU_i^\top + \hat{\bm{\Sigma}}_i$\;
            $\btV_i^\top \gets (\bV_i^\top - \bT_i\bm{\Psi}_i\bU_i^\top)\btD_i^{-1}$\;
            $\boT_i \gets \bT_i - \btV_i^\top\bU_i$\;
            $\bA_i \gets \btD_i^{-1}\hat{\bm{\Sigma}}_i$\;
            $\bE_i \gets \btD_i^{-1}$
        }{
            $\btV_i^\top \gets \bzero$, $\boT_i \gets \bT_i$, $\bA_i \gets \bI$, $\bE_i \gets \bzero$, $\btD_i \gets \bzero$\;
        }
            $\bm{\Psi}_{i+1} \gets \bT_i\bm{\Psi}_i\bT_i^\top + \btV_i^\top\btD_i\btV_i$\;
    }

    $\bm{\Phi} \gets \bzero_{p\times p}$\;
    \For{$i = N$ \KwTo $1$}{
        $\bW_i^\top \gets (\bV_i^\top - \bT_i\bm{\Psi}_i\bU_i^\top)\bA_i$\;
        $\hat{\bm{\Omega}}_i \gets (\bD_i - \bU_i\bm{\Psi}_i\bU_i^\top)\bA_i - \bW_i\bm{\Phi}\bW_i^\top$\;
        $\bm{\Phi} \gets \boT_i^\top\bm{\Phi}\boT_i + \bU_i^\top\bE_i\bU_i$\;
    }
    \Return{$(\hat{\bm{\Omega}}_1,\ldots,\hat{\bm{\Omega}}_N)$}
\end{algorithm}

\subsection{Log marginal likelihood}\label{app:kalman:marginal}
Recall the log marginal likelihood function
\begin{equation*}
	\log p(\by) = -\frac{1}{2}\by^\top(\bK + \bm{\Sigma})^{-1}\by - \frac{1}{2}\log \det(\bK + \bm{\Sigma}) - \frac{nd}{2}\log(2\pi).
\end{equation*}
Letting $\bK+\bm{\Sigma} = \bL\btD\bL^\top$ and $\bz=\bL^{-1}\by$, we can rewrite the log marginal likelihood as
\begin{align*}
		\log p(\by)
		 & = -\frac{1}{2}\sum_{i=1}^n \bz_i^\top\btD_i^{-1}\bz_i - \frac{1}{2}\sum_{i=1}^n \log \det(\btD_i) - \frac{nd}{2}\log(2\pi),
\end{align*}
which coincides with the filtering marginal likelihood from \cite{sarkka:kalman}. The subvector $\bz_i$ may be viewed as the $i$th  innovation with covariance $\btD_i$.

\section{Computing the stationary covariance}\label{app:computing:Pinf}

Consider the continuous-time state-space model in \eqref{eq:sde}, and assume that $\bF$ is a Hurwitz matrix of order $m+1$ in companion form, i.e.,
\[ \bF = \begin{bmatrix}
		0    & 1                          \\
		     & \ddots & \ddots            \\
		     &        & 0        & 1      \\
		-c_0 & \cdots & -c_{m-1} & -c_{m}
	\end{bmatrix} = \bN^\top - \be_{m+1}\bc^\top \]
where $\bc = (c_0,\ldots,c_{m})$ are the coefficients of the characteristic polynomial of $\bF$, and $\bN$ is the lower shift matrix of order $m+1$.
Recall that the stationary state covariance $\bP_\infty$ is the solution to the Lyapunov equation
\[ \bF \bP_\infty + \bP_\infty \bF^\top + \bG \bQ_{\bw} \bG^\top = 0. \]
With our assumptions on $\bF$, and letting $\bP = \bP_\infty$, this equation can be rewritten as
\begin{equation}
	(\bN^\top - \be_{m+1}\bc^\top)\bP + \bP(\bN - \bc \be_{m+1}^\top) + \bG \bQ_{\bw} \bG^\top = 0,
\end{equation}
the $(i,j)$th element of which can be expressed as
\begin{align*}
	\be_i^\top(\bN^\top - \be_{m+1}\bc^\top)\bP\be_j + \be_i^\top\bP(\bN - \bc \be_{m+1}^\top)\be_j + \be_i^\top \bG \bQ_{\bw} \bG^\top \be_j = 0.
\end{align*}
Noting that
\begin{align*}
	\bN \be_i = \begin{cases}
		            \be_{i+1}, & 1 \leq i \leq m, \\
		            0,         & i = m+1,
	            \end{cases}
	\qquad \text{and} \qquad
	\bN^\top \be_i =
	\begin{cases}
		0,         & i = 1,               \\
		\be_{i-1}, & 2 \leq i \leq {m+1},
	\end{cases}
\end{align*}
we see that for $i,j \in \{1,\ldots,m\}$, the $(i,j)$th element of the Lyapunov equation reduces to
\begin{align}\label{eq:lyapunov:ij}
	\bP_{i+1,j} + \bP_{i,j+1} = -\be_i^\top \bG \bQ_{\bw} \bG^\top \be_j,
\end{align}
and similarly, for $i = m+1$, we have
\begin{align*}
	\begin{cases}
		\bP_{m+1,j+1} - \bc^\top\bP\be_j = -\be_{m+1}^\top \bG \bQ_{\bw} \bG^\top \be_j,                 & 1\leq j\leq m, \\
		\bc^\top\bP\be_{m+1} + \be_{m+1}^\top \bP \bc = \be_{m+1}^\top \bG \bQ_{\bw} \bG^\top \be_{m+1}, & j = m+1.
	\end{cases}
\end{align*}

For special cases of the input matrix $\bG$ and the spectral density $\bQ_{\bw}$, the Lyapunov equation simplifies further. For example, if $\bG \bQ_{\bw} \bG^\top = q_w \be_{m+1}\be_{m+1}^\top$ as in the Matérn  process, the condition \eqref{eq:lyapunov:ij} reduces to
\begin{align*}\label{eq:lyapunov:ij:special}
	\bP_{i+1,j} + \bP_{i,j+1} = 0, \quad i,j\in\{1,\ldots,m\}.
\end{align*}
This implies that the entries along the antidiagonals of $\bP$ have the same absolute value but alternating signs, and by symmetry, the entries of $\bP$ are fully determined by the diagonal entries $\{\bP_{ii}\}_{i=1}^{m+1}$. For example, when $m=4$, $\bP$ takes the form
\[ \bP = \begin{bmatrix}
		p_1  & 0    & -p_2 & 0    & p_3  \\
		0    & p_2  & 0    & -p_3 & 0    \\
		-p_2 & 0    & p_3  & 0    & -p_4 \\
		0    & -p_3 & 0    & p_4  & 0    \\
		p_3  & 0    & -p_4 & 0    & p_5
	\end{bmatrix}.
\]
More generally, we can parameterize $\bP$ in terms of the shift matrix $\bN$ and a diagonal matrix $\Mf=\diag(p_1,\ldots,p_{m+1})$ such that
\begin{align}
	\bP = \Mf + \sum_{k=1}^{\lfloor m/2\rfloor} (-1)^k \left[ \bN^k \Mf \bN^k + (\bN^k \Mf \bN^k)^\top \right],
\end{align}
and hence the Lyapunov equation reduces to a linear system of $m+1$ equations in the diagonal entries $\{p_i\}_{i=1}^{m+1}$,
\begin{align*}
	\begin{cases}
		\bP_{m+1,j+1} = \bc^\top\bP\be_j = 0,                & 1\leq j\leq m, \\
		\bc^\top\bP\be_{m+1} + \be_{m+1}^\top \bP \bc = q_w, & j = m+1.
	\end{cases}
\end{align*}

\section{Diagonal blocks of the posterior covariance}\label{app:posterior:variance}
We now derive \Cref{alg:posterior:variance}, starting from the expression \eqref{eq:gp:posterior:cov:perturbed}, i.e.,
\begin{align*}
	\hat \bK -\hat \bK (\hat \bK + \hat{\bm{\Sigma}}(\gamma))^{-1} \hat \bK = \hat{\bm{\Sigma}}(\gamma) -\hat{\bm{\Sigma}}(\gamma)\hat{\bM}(\gamma)^{-1}\hat{\bm{\Sigma}}(\gamma),
\end{align*}
with $\hat{\bM}(\gamma) = \hat{\bK} + \hat{\bm{\Sigma}}(\gamma)$, where the SSS matrices $\hat{\bK}$ and $\hat{\bm{\Sigma}}$ are defined in \Cref{sub:sec:posterior}. We are interested in the diagonal blocks of this expression in the limit as $\gamma\rightarrow\infty$. For $\gamma$ positive and finite, $\hat{\bM}(\gamma)$ is positive definite, and we can compute its Cholesky factorization $\hat{\bM}(\gamma) = \hat{\bL}(\gamma)\hat{\bL}(\gamma)^\top$ and the corresponding SSS representation of $\hat{\bL}(\gamma)^{-1}$ using \Cref{alg:cholesky,alg:explicit:inverse}, i.e.,
\begin{align*}
	\hat{\bL}(\gamma)      & = \bU(\bI - \bN\bT)^{-1}\bN\btV^\top + \btD,   &
	\hat{\bL}(\gamma)^{-1} & = \boU(\bI - \bN\boT)^{-1}\bN\boV^\top + \boD.
\end{align*}
We note that the diagonal blocks of $\hat{\bM}(\gamma)$ are of the form $\bD_i(\gamma) = \hat{\bK}_{ii} + \hat{\bm{\Sigma}}_i(\gamma)$, and hence the SSS generators $\btV$ and $\btD$, as well as those of $\hat{\bL}(\gamma)^{-1}$, implicitly depend on $\gamma$.
The $i$th diagonal block of \eqref{eq:gp:posterior:cov:perturbed} can then be expressed as
\begin{align}\label{eq:gp:posterior:cov:perturbed:block}
	\hat{\bm{\Omega}}_i(\gamma) = \hat{\bm{\Sigma}}_i(\gamma) - \hat{\bm{\Sigma}}_i(\gamma) \bm{\Omega}_i(\gamma) \hat{\bm{\Sigma}}_i(\gamma),
\end{align}
where $\bm{\Omega}_i(\gamma)$ is the $i$th diagonal block of $\hat{\bM}(\gamma)^{-1}$ and $\hat{\bm{\Sigma}}_i(\gamma)$ denotes the $i$th diagonal block of $\hat{\bm{\Sigma}}(\gamma)$. From \eqref{eq:inverse:diagonal}, we have
\begin{align*}
	\bm{\Omega}_i(\gamma) = \boV_i(\gamma) \bm{\Phi}_i(\gamma) \boV_i(\gamma)^\top + \boD_i(\gamma)^\top \boD_i(\gamma),
\end{align*}
where we have made the dependence on $\gamma$ explicit. Inserting this into \eqref{eq:gp:posterior:cov:perturbed:block}, and using the fact that $\boD_i(\gamma) = \btD_i(\gamma)^{-1}$, we can express the $i$th diagonal block of \eqref{eq:gp:posterior:cov:perturbed} as
\begin{align}
	\notag
	\hat{\bm{\Omega}}_i(\gamma) & =	\hat{\bm{\Sigma}}_i(\gamma) - \hat{\bm{\Sigma}}_i(\gamma) \bm{\Omega}_i(\gamma) \hat{\bm{\Sigma}}_i(\gamma)                                                      \\
	\notag
	                            & =
	\hat{\bm{\Sigma}}_i(\gamma) - \hat{\bm{\Sigma}}_i(\gamma) \left( \boV_i(\gamma) \bm{\Phi}_i(\gamma) \boV_i(\gamma)^\top + \boD_i(\gamma)^\top \boD_i(\gamma) \right) \hat{\bm{\Sigma}}_i(\gamma) \\
	\label{eq:gp:posterior:cov:blocks:gamma}
	                            & =
	(\btD_i(\gamma)\btD_i(\gamma)^\top - \hat{\bm{\Sigma}}_i(\gamma))\boD_i(\gamma)^\top \boD_i(\gamma)\hat{\bm{\Sigma}}_i(\gamma) \journallayout{\\ & \quad}{}-\hat{\bm{\Sigma}}_i(\gamma) \boV_i(\gamma) \bm{\Phi}_i(\gamma) \boV_i(\gamma)^\top \hat{\bm{\Sigma}}_i(\gamma). \journallayout{\notag}{}
\end{align}
From \eqref{eq:diag}, we find that
\begin{align*}
	\btD_i(\gamma)\btD_i(\gamma)^\top - \hat{\bm{\Sigma}}_i(\gamma) = \bD_i(\gamma) - \hat{\bm{\Sigma}}_i(\gamma) - \bU_i \bm{\Psi}_i(\gamma) \bU_i^\top =	\hat{\bK}_{ii} - \bU_i \bm{\Psi}_i(\gamma) \bU_i^\top,
\end{align*}
and from the definition of $\boV_i^\top$ in \eqref{eq:inverse:ssm} and \eqref{eq:offdiag}, we have
\begin{align*}
	\boV_i(\gamma)^\top = (\bV_i^\top - \bT_i\bm{\Psi}_i(\gamma)\bU_i^\top) \boD_i(\gamma)^{\top} \boD_i(\gamma) ).
\end{align*}
Using these expressions and letting $\bA_i(\gamma) = \boD_i(\gamma)^{\top} \boD_i(\gamma) \hat{\bm{\Sigma}}_{i}(\gamma)$, we can express \eqref{eq:gp:posterior:cov:blocks:gamma} as
\begin{align*}
	\hat{\bm{\Omega}}_i(\gamma) & = (\hat{\bK}_{ii} - \bU_i \bm{\Psi}_i(\gamma) \bU_i^\top) \bA_i(\gamma)
	\journallayout{\\ & \quad}{}- \bA_i(\gamma)^\top (\bV_i^\top - \bT_i\bm{\Psi}_i(\gamma)\bU_i^\top)^\top \bm{\Phi}_i(\gamma) (\bV_i^\top - \bT_i\bm{\Psi}_i(\gamma)\bU_i^\top) \bA_i(\gamma).
\end{align*}
To find the limit as $\gamma \to \infty$, we first examine the asymptotic behavior for the SSS generators of $\hat{\bL}(\gamma)$. From \eqref{eq:diag}, we have
\begin{align*}
	\btD_{i}(\gamma) =
	\begin{cases}
		\chol(\hat{\bK}_{ii}  -\bU_i \bm{\Psi}_i(\gamma)\bU_i^\top + \hat{\bm{\Sigma}}_i), & i \in \cI,   \\
		\chol(\hat{\bK}_{ii}  -\bU_i \bm{\Psi}_i(\gamma)\bU_i^\top + \gamma \bI),          & i \in \cI_*,
	\end{cases}
\end{align*}
and from \eqref{eq:cholesky:psi:recursion} and \eqref{eq:offdiag}, we have $\bm{\Psi}_{i+1}(\gamma) = \bT_i\bm{\Psi}_i(\gamma)\bT_i^\top + \btV_i(\gamma)^\top\btV_i(\gamma)$ and $\btV_i^\top = (\bV_i^\top - \bT_i\bm{\Psi}_i(\gamma)\bU_i^\top) \btD_i^{-\top}$. Thus, the SSS generators of $\hat{\bL}(\gamma)$ satisfy the asymptotic equivalences
\begin{align*}
	\btD_{i}(\gamma)    & \sim \begin{cases}
		                           \chol(\hat{\bK}_{ii}  -\bU_i \bm{\Psi}_i(\gamma)\bU_i^\top + \hat{\bm{\Sigma}}_i), & i \in\cI,    \\
		                           \gamma^{1/2} \, \bI,                                                               & i \in \cI_*,
	                           \end{cases} \journallayout{\\}{&}
	\btV_i(\gamma)^\top & \sim \begin{cases}
		                           (\bV_i^\top - \bT_i\bm{\Psi}_i(\gamma)\bT_i^\top) \btD_i(\gamma)^{-\top}, & i \in \cI,   \\
		                           \gamma^{-1/2} (\bV_i^\top - \bT_i\bm{\Psi}_i(\gamma)\bU_i^\top),          & i \in \cI_*,
	                           \end{cases}
\end{align*}
and similarly, for $i \in \{1,\ldots,n-1\}$,
\begin{align*}                                                                                          \\
	\bm{\Psi}_{i+1}(\gamma) & \sim \begin{cases}
		                               \bT_i\bm{\Psi}_i(\gamma)\bT_i^\top + \btV_i(\gamma)^\top\btV_i(\gamma),                                                                                      & i\in\cI,     \\
		                               \bT_i\bm{\Psi}_i(\gamma)\bT_i^\top + \gamma^{-1}(\bV_i^\top -  \bT_i \bm{\Psi}_i(\gamma)\bU_i^\top) (\bV_i^\top - \bT_i \bm{\Psi}_i(\gamma)\bU_i^\top)^\top, & i \in \cI_*.
	                               \end{cases}
\end{align*}
We note that $\bm{\Psi}_{i+1}(\gamma)$ is symmetric positive semidefinite by construction, and its limit as $\gamma \to \infty$ is well-defined and finite. This follows by induction on $i$. Suppose $\bm{\Psi}_i(\gamma)$ is finite and positive semidefinite for all $\gamma > 0$. This is true for $i=1$, since $\bm{\Psi}_1(\gamma) = \bzero$. Then, by continuity, the recursive definition of $\bm{\Psi}_{i+1}(\gamma)$ implies that $\bm{\Psi}_{i+1}(\gamma)$ is also finite and positive semidefinite for all $\gamma > 0$. We then define $\bm{\Psi}_{i}(\infty) = \lim_{\gamma \to \infty} \bm{\Psi}_{i}(\gamma)$ and
\begin{align*}
	\boD_i(\infty) = \lim_{\gamma\to \infty} \btD_i(\gamma)^{-1} =
	\begin{cases}
		\chol(\hat{\bK}_{ii} -\bU_i\bm{\Psi}_i(\infty)\bU_i^\top + \hat{\bm{\Sigma}}_i)^{-1}, & i\in\cI,     \\
		\bzero,                                                                               & i \in \cI_*,
	\end{cases}
\end{align*}
and
\begin{align*}
	\btV_i(\infty)^\top = \lim_{\gamma \to \infty} \btV_i(\gamma)^\top = (\bV_i^\top - \bT_i\bm{\Psi}_i(\infty)\bU_i^\top) \boD_i(\infty)^{\top} , \quad i \in \{1,\ldots,N\}.
\end{align*}
This leads to
\begin{align}
	\label{eq:psi:infty}
	\bm{\Psi}_{i+1}(\infty) = \bT_i \bm{\Psi}_{i}(\infty) \bT_i^\top + \btV_i(\infty)^\top \btV_i(\infty), \quad i \in \{1,\ldots,N-1\},
\end{align}
where $\bm{\Psi}_1(\infty) = \bzero$. Using \eqref{eq:inverse:ssm} and \eqref{eq:offdiag}, we find that
\[\boT_i(\infty) = \lim_{\gamma\to\infty} \boT_i(\gamma) =  \bT_i - \boV_i(\infty)^\top \bU_i = \bT_i + \btV_i(\infty)^\top\boU_i(\infty),\]
and hence
\begin{align*}
	\bm{\Phi}_{i-1}(\infty) & = \boT_i^\top(\infty)  \bm{\Phi}_{i}(\infty) \boT_i(\infty) + \bU_i^\top\boD_i(\infty)^\top\boD_i(\infty) \bU_i
\end{align*}
Finally, we have
\begin{align*}
	\bA_i(\infty) = \lim_{\gamma\to \infty} \boD_i(\gamma)^\top \boD_i(\gamma)\hat{\bm{\Sigma}}_i(\gamma) = \begin{cases}
		                                                                                                        \boD_i(\infty)^\top \boD_i(\infty) \hat{\bm{\Sigma}}_{i}, & i \in \cI,   \\
		                                                                                                        \bI,                                                      & i \in \cI_*.
	                                                                                                        \end{cases}
\end{align*}
Combining these results, we can compute the diagonal blocks of the posterior covariance as
\begin{align*}
    \hat{\bm{\Omega}}_i(\infty) & = \lim_{\gamma\to\infty} \hat{\bm{\Omega}}_i(\gamma) \\
                                & = (\hat{\bK}_{ii} - \bU_i \bm{\Psi}_i(\infty) \bU_i^\top) \bA_i(\infty) \journallayout{\\ & \quad}{}- \bA_i(\infty)^\top (\bV_i^\top - \bT_i\bm{\Psi}_i(\infty)\bU_i^\top)^\top \bm{\Phi}_i(\infty) (\bV_i^\top - \bT_i\bm{\Psi}_i(\infty)\bU_i^\top) \bA_i(\infty),
\end{align*}
where $\bm{\Psi}_i(\infty)$, $\btV_i(\infty)$, $\boD_i(\infty)$, $\bm{\Phi}_i(\infty)$, and $\bA_i(\infty)$ are computed using the above expressions. This leads to \Cref{alg:posterior:variance}.

\section{Generalized Cross Validation}\label{app:gcv}
The GCV score can be expressed as
\begin{align}\label{eq:gp:gcv}
    \text{GCV}(\btheta) = \frac{\frac{1}{nd}\|(\bI-\bH)\by\|_2^2}{(1-\tr(\bH)/(nd))^2} = nd\,\frac{\|(\bI-\bH)\by\|_2^2}{(nd-\tr(\bH))^2}, \qquad \bH = \bK(\bK + \bm{\Sigma})^{-1},
\end{align}
and its partial derivatives with respect to the hyperparameters can be computed analytically as
\begin{align}\label{eq:gp:gcv:grad}
    \frac{\partial}{\partial \theta_k} \text{GCV}(\btheta) = nd\left[-\frac{2\by^\top(\bI-\bH)^\top \frac{\partial \bH}{\partial \theta_k}\by}{(nd-\tr(\bH))^2} + \frac{2\|(\bI-\bH)\by\|_2^2 \tr\!\left(\frac{\partial \bH}{\partial \theta_k}\right)}{(nd-\tr(\bH))^3} \right],
\end{align}
where
\begin{align} \label{eq:gp:gcv:grad:H}
\frac{\partial \bH}{\partial \theta_k} = \frac{\partial \bK}{\partial \theta_k} (\bK+\bm{\Sigma})^{-1} -\bK(\bK+\bm{\Sigma})^{-1} \frac{\partial (\bK+\bm{\Sigma})}{\partial \theta_k} (\bK+\bm{\Sigma})^{-1}.
\end{align}
The GCV score is readily evaluated in linear time using \Cref{alg:cholesky,alg:trace}. The partial derivatives of the GCV score can also be evaluated in linear time using \Cref{alg:cholesky,alg:trace,alg:diagonal:inverse}, where we note that the last term on the right-hand side of \eqref{eq:gp:gcv:grad:H} is the product of SSS matrices, which can be computed in linear time using the algorithm from \cite{chandrasekaran2003b}.

\end{document}